%% file: main.tex
\documentclass[a4paper, table]{article}
\input{General/Preamble} 
\input{General/Settings} 

\title{\name{}: Language Modeling with Explicit Memory}
\date{July 1, 2024}

\author[1]{Hongkang Yang}
\author[1]{Zehao Lin}
\author[1]{Wenjin Wang}
\author[1]{Hao Wu}
\author[1]{Zhiyu Li}
\author[1]{Bo Tang}
\author[1]{Wenqiang Wei}
\author[1]{Jinbo Wang}
\author[1]{Zeyun Tang}
\author[1]{Shichao Song}
\author[1]{Chenyang Xi}
\author[1]{Yu Yu}
\author[1]{Kai Chen}
\author[1]{Feiyu Xiong}
\author[2]{Linpeng Tang}
\author[3,1]{Weinan E\thanks{Also at School of Mathematical Sciences, Peking University and AI for Science Institute\\Corresponding authors: xiongfy@iaar.ac.cn, linpengt@myscale.com, weinan@math.pku.edu.cn}}

\affil[1]{Center for LLM, Institute for Advanced Algorithms Research, Shanghai}
\affil[2]{Moqi Inc}
\affil[3]{Center for Machine Learning Research, Peking University}

\begin{document}

\newgeometry{top=1cm, bottom=1.4cm, left=2.5cm, right=2.5cm}

\maketitle



\begin{abstract}
The training and inference of large language models (LLMs) are together a costly process that transports knowledge from raw data to meaningful computation.
Inspired by the memory hierarchy of the human brain, we reduce this cost by equipping LLMs with explicit memory, a memory format cheaper than model parameters and text retrieval-augmented generation (RAG).
Conceptually, with most of its knowledge externalized to explicit memories, the LLM can enjoy a smaller parameter size, training cost, and inference cost, all proportional to the amount of remaining ``abstract knowledge".
As a preliminary proof of concept, we train from scratch a 2.4B LLM, which achieves better performance than much larger LLMs as well as RAG models, and maintains higher decoding speed than RAG.
The model is named \name{}, since explicit memory is the third form of memory in LLMs after implicit memory (model parameters) and working memory (context key-values).
We introduce a memory circuitry theory to support the externalization of knowledge, and present novel techniques including a memory sparsification mechanism that makes storage tractable and a two-stage pretraining scheme that facilitates memory formation.
\end{abstract}

\begin{figure}[H]
\centering
\includegraphics[width=0.85\textwidth]{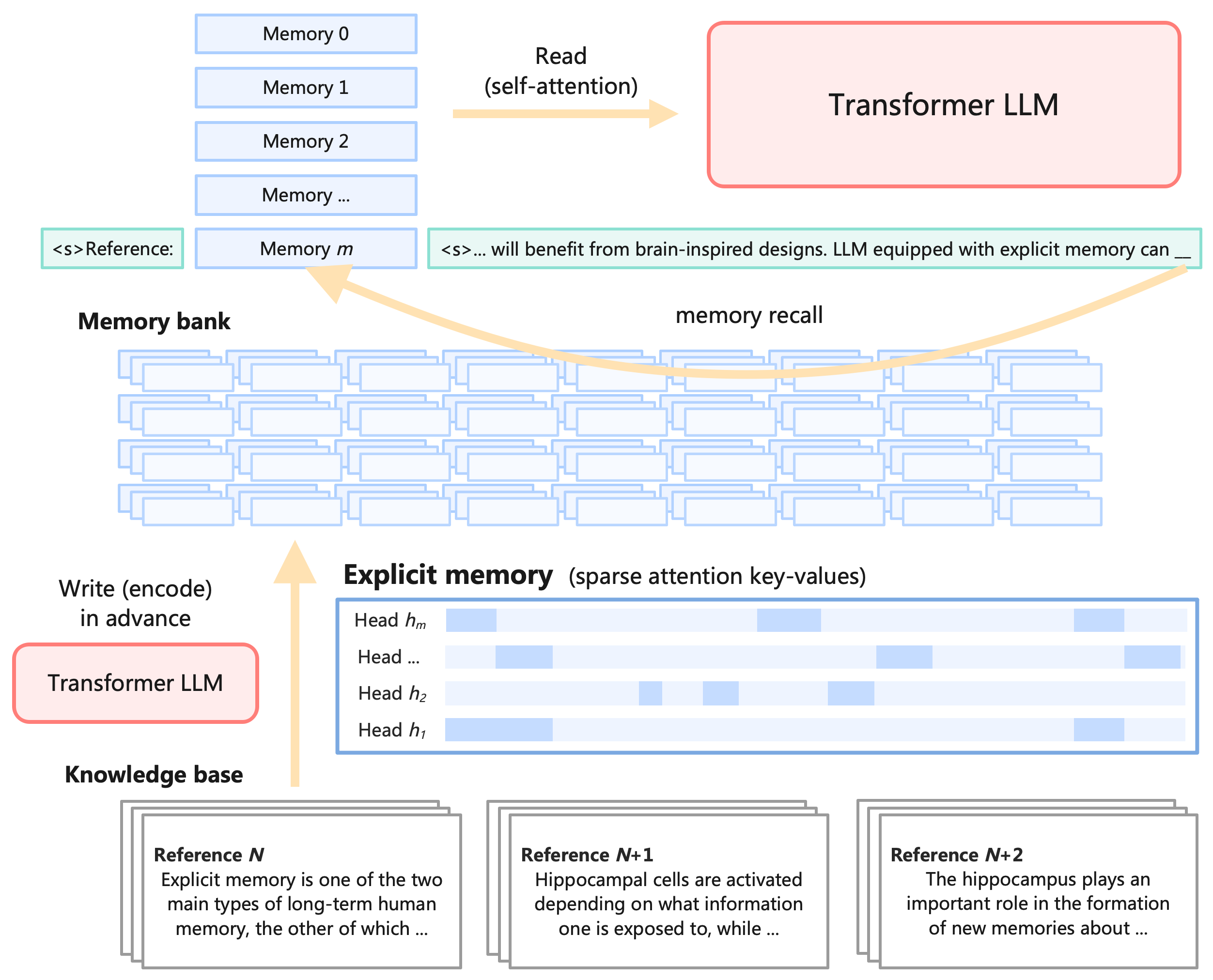}
\caption{The \name{} model converts texts to explicit memories, and then recalls these memories during inference.
The explicit memories can be seen as retrievable model parameters, externalized knowledge, or sparsely-activated neural circuits.}
\label{fig:opening}
\end{figure}

\newpage
\newgeometry{margin=2.5cm}

\begin{figure}[H]
\centering
\subfloat{\includegraphics[width=0.48\textwidth]{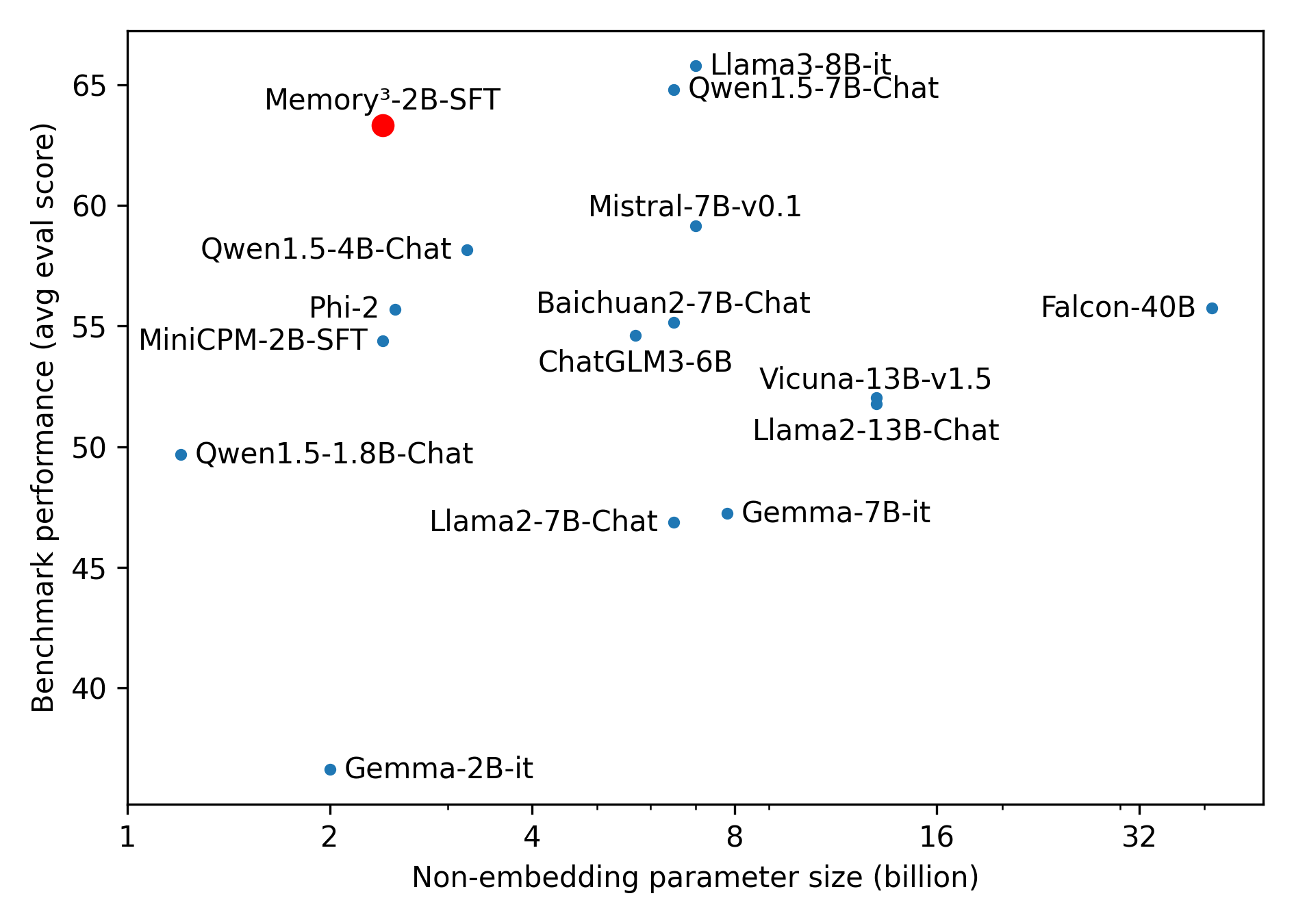}}
\quad
\subfloat{\includegraphics[width=0.48\textwidth]{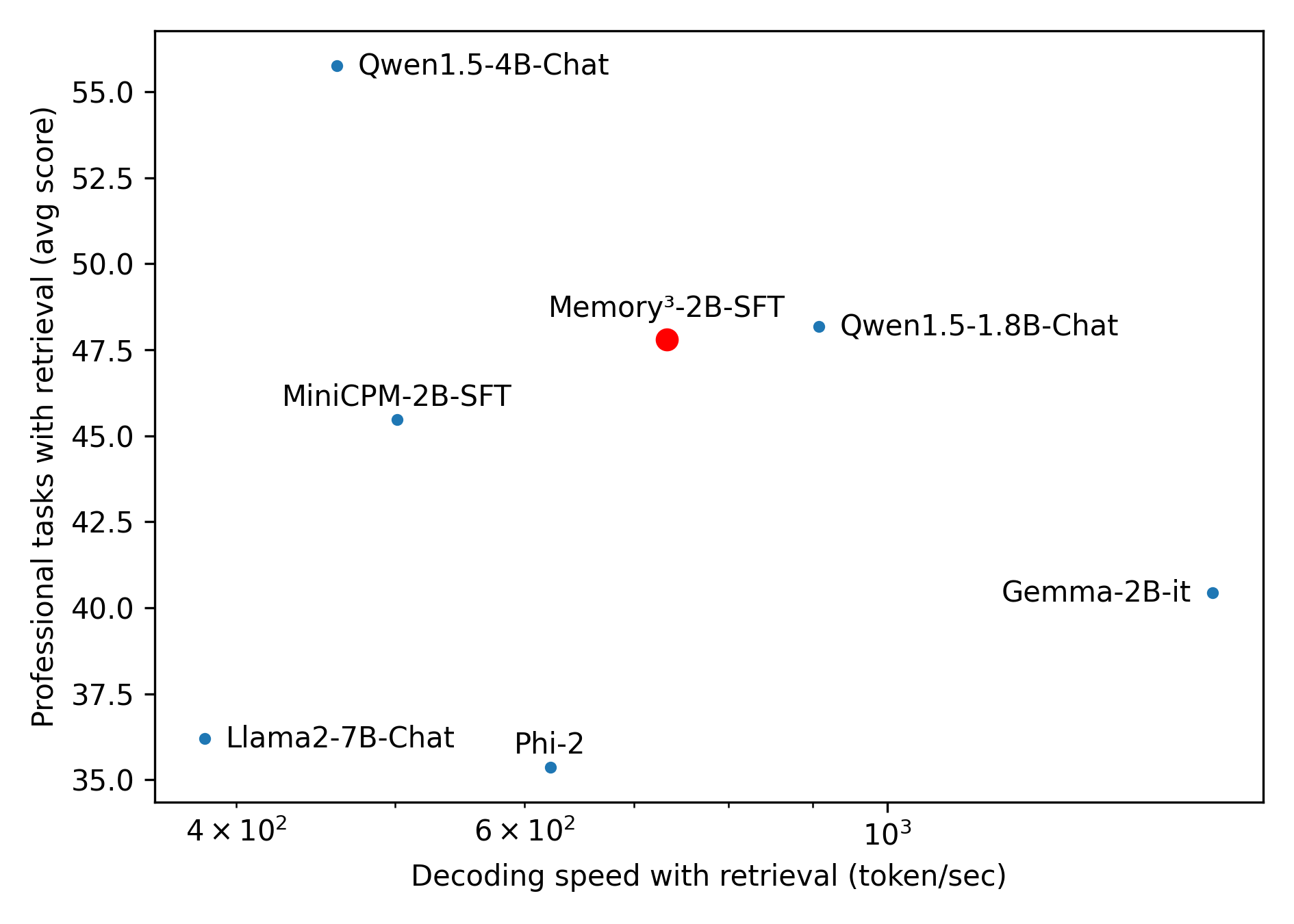}}
\caption{
Left: Performance on benchmarks, with respect to model size (top-left is better).
Right: Retrieval-augmented performance on professional tasks, versus decoding speed with retrieval (top-right is better).
The left plot is based on Table \ref{tab:basic-results}.
The right plot is based on Tables \ref{tab:domain-results} and \ref{tab:throughput}.
\name{} uses high frequency retrieval of explicit memories, while the RAG models use a fixed amount of 5 references.
This is a preliminary experiment and we have not optimized the quality of our pretraining data as well as the efficiency of our inference pipeline, so the results may not be comparable to those of the SOTA models.}
\label{fig:eval plots}
\end{figure}

\section{Introduction}
\label{sec:intro}

Large language models (LLMs) have enjoyed unprecedented popularity in recent years thanks to their extraordinary performance \cite{openai2023gpt4,anthropic2024claude,touvron2023llama,bai2023qwen,yang2023baichuan,abdin2024phi3,jiang2023mistral,hu2024minicpm}.
The prospect of scaling laws \cite{kaplan2020scaling,hoffmann2022training,ruan2024observational} and emergent abilities \cite{wei2022emergent,srivastava2023imitation} constantly drives for substantially larger models, resulting in the rapid increase in the cost of LLM training and inference.
People have been trying to reduce this cost through optimizations in various aspects, including architecture \cite{fedus2022switch,ainslie2023gqa,dai2019transformerxl,liu2023dejavu,peng2023rwkv,sun2024you}, data quality \cite{cerebras2023slimpajama,kaddour2023minipile,gunasekar2023textbooks,li2023textbooks}, operator \cite{dao2022flashattention,kwon2023paged}, parallelization \cite{rasley2020deepspeed,shoeybi2020megatronlm,korthikanti2022reducing,qi2023zero}, optimizer \cite{liu2024sophia,xie2023adan,wang2024IRE}, scaling laws \cite{hoffmann2022training,yang2022tensor}, generalization theory \cite{zhang2024initialization,huang2024unified}, hardware \cite{dey2023cerebrasgpt}, etc.

We introduce the novel approach of optimizing knowledge storage.
The combined cost of LLM training and inference can be seen as the cost of encoding the knowledge from text data into various memory formats, plus the cost of reading from these memories during inference:
\begin{equation}
\label{eq:cost}
\sum_{\text{knowledge }k} \min_{\text{format }m} \text{cost}_{\text{write}}(k, m) + n_k \cdot \text{cost}_{\text{read}}(k, m)
\end{equation}
where $\text{cost}_{\text{write}}$ is the cost of encoding a piece of knowledge $k$ into memory format $m$,
$\text{cost}_{\text{read}}$ is the cost of integrating $k$ from format $m$ into inference,
and $n_k$ is the expected usage count of this knowledge during the lifespan of this LLM (e.g. a few months for each version of ChatGPT \cite{openai2024version,azure2024version}).
The definitions of knowledge and memory in the context of LLMs are provided in Section \ref{sec:theory}, and this paper uses knowledge as a countable noun.
Typical memory formats include model parameters and plain text for retrieval-augmented generative models (RAG); their write functions and read functions are listed in Table \ref{table:memory}, and their $\text{cost}_{\text{write}}$ and $\text{cost}_{\text{read}}$ are provided in Figure \ref{fig:total_cost_usage_2B}.

We introduce a new memory format, explicit memory, characterized by moderately low write cost and read cost.
As depicted in Figure \ref{fig:opening}, our model first converts a knowledge base (or any text dataset) into explicit memories, implemented as sparse attention key-values, and then during inference, recalls these memories and integrates them into the self-attention layers.
Our design is simple so that most of the existing Transformer-based LLMs should be able to accommodate explicit memories with a little finetuning, and thus it is a general-purpose ``model amplifier".
Eventually, it should reduce the cost of pretraining LLMs, since there will be much less knowledge that must be stored in parameters, and thus less training data and smaller model size.

The new memory format enables us to define a memory hierarchy for LLMs:
\begin{center}
plain text (RAG) $\to$ explicit memory $\to$ model parameter
\end{center}
such that by going up the hierarchy, $\text{cost}_{\text{write}}$ increases while $\text{cost}_{\text{read}}$ decreases.
To minimize the cost (\ref{eq:cost}), one should store each piece of knowledge that is very frequently/rarely used in the top/bottom of this hierarchy, and everything in between as explicit memory.
As illustrated in Table \ref{table:memory}, the memory hierarchy of LLMs closely resembles that of humans.
For humans, the explicit/implicit memories are the long-term memories that are acquired and used consciously/unconsciously \cite{kandel2021principles}.

\begin{table}[h]
\centering
\small
\rowcolors{2}{Tue-red!10}{white}
\begin{tabular}{ | c | c | c | c | c |}
\hline
\makecell{Memory format\\of humans} & Example & \makecell{Memory format\\of LLMs} & Write & Read \\
\hline\hline
Implicit memory & common expressions & model parameters & training & matrix multiplication \\
\hline
Explicit memory & books read & this work & memory encoding & self-attention \\
\hline
External information & open-book exam & plain text (RAG) & none & encode from scratch \\
\hline
\end{tabular}
\caption{Analogy of the memory hierarchies of humans and LLMs.}
\label{table:memory}
\end{table}

As a remark, one can compare the plain LLMs to patients with impaired explicit memory, e.g. due to injury to the medial temporal lobe.
These patients are largely unable to learn semantic knowledge (usually stored as explicit memory), but can acquire sensorimotor skills through repetitive priming (stored as implicit memories) \cite{gabrieli1988impaired,corkin2002s,bayley2005failure}.
Thus, one may hypothesize that due to the lack of explicit memory, the training of plain LLMs is as inefficient as repetitive priming, and thus has ample room for improvement.
In analogy with humans, for instance, it is easy to recall and talk about a book we just read, but to recite it as unconsciously as tying shoe laces requires an enormous effort to force this knowledge into our muscle memory.
From this perspective, it is not surprising that LLM training consumes so much data and energy \cite{wu2022sustainable,luccioni2022carbon}.
We want to rescue LLMs from this poor condition by equipping it with an explicit memory mechanism as efficient as that of humans.

\begin{figure}[h]
\centering
\includegraphics[width=0.7\textwidth]{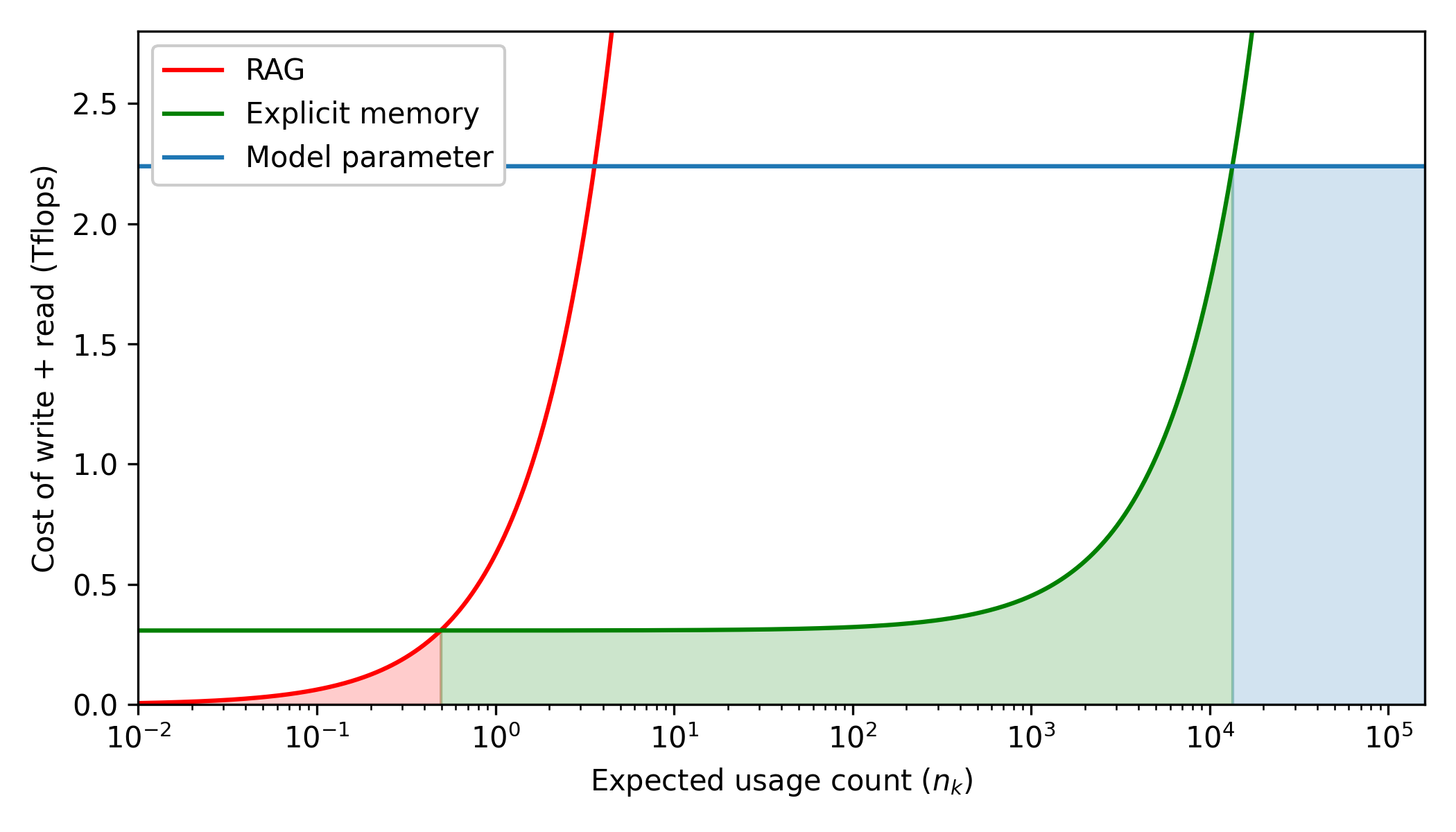}
\caption{The total cost (TFlops) of writing and reading a piece of knowledge by our 2.4B model with respect to its expected usage count.
The curves represent the cost of different memory formats, and the shaded area represents the minimum cost given the optimal format.
The plot indicates that $(0.494, 13400)$ is the advantage interval for explicit memory.
The calculations are provided in Appendix \ref{appendix:cost}.
(The blue curve is only a lower bound on the cost of model parameters.)}
\label{fig:total_cost_usage_2B}
\end{figure}

A quantitative illustration of the cost (\ref{eq:cost}) is given by Figure \ref{fig:total_cost_usage_2B}, where we characterize $\text{cost}_{\text{write}}$ and $\text{cost}_{\text{read}}$ by the amount of compute (TFlops).
The plot indicates that if a piece of knowledge has an expected usage count $\in (0.494, 13400)$, then it is optimal to be stored as an explicit memory.
Moreover, the introduction of explicit memory helps to externalize the knowledge stored in model parameters and thus allow us to use a lighter backbone, which ultimately reduces all the costs in Figure \ref{fig:total_cost_usage_2B}.

\vspace{0.5em}
The second motivation for explicit memory is to alleviate the issue of knowledge traversal.
Knowledge traversal happens when the LLM wastefully invokes all its parameters (and thus all its knowledge) each time it generates a token.
As an analogy, it is unreasonable for humans to recall everything they learned whenever they write a word.
Let us define the knowledge efficiency of an LLM as the ratio of the minimum amount of knowledge sufficient for one decoding step to the amount of knowledge actually used.
An optimistic estimation of knowledge efficiency for a 10B LLM is $10^{-5}$:
On one hand, it is unlikely that generating one token would require more than $10^4$ bits of knowledge (roughly equivalent to a thousand-token long passage, sufficient for enumerating all necessary knowledge);
on the other hand, each parameter is involved in the computation and each stores at least 0.1 bit of knowledge \cite[Result 10]{allenzhu2024physics} (this density could be much higher if the LLM is trained on cleaner data), thus using $10^9$ bits in total.

A novel architecture is needed to boost the knowledge efficiency of LLMs from $10^{-5}$ to $1$, whereas current designs are far from this goal.
Consider the mixture-of-experts architecture (MoE) for instance, which uses multiple MLP layers (experts) in each Transformer block and process each token with only a few MLPs.
The boost of MoE, namely the ratio of the total amount of parameters to the amount of active parameters, is usually bounded by $4\sim 32$ \cite{fedus2022switch,jiang2023mistral,snowflake2024arctic}.
Similarly, neither the mixture-of-depth architecture \cite{elbayad2020depthadaptive,raposo2024mixtureofdepths} nor sparsified MLP neurons and attention heads \cite{liu2023dejavu} can bring greater gains.
RAG appears very sparse if we compare the amount of retrieved texts with the size of the text database; nevertheless, RAG is usually built upon a plain LLM as backbone, which provides most of the knowledge used in inference, and thus offers little assistance in addressing the knowledge traversal problem.

An ideal solution is to retrieve only the needed parameters for each token.
This is naturally achieved by explicit memories if we compare memory recall to parameter retrieval.

\vspace{0.5em}
The third motivation is that, as a human-like design, explicit memory enables LLMs to develop more human-like capabilities.
To name a few,
\begin{itemize}
\item Infinitely long context:
LLMs have the difficulty of processing long texts since their working memory (context key-values) costs too much GPU memory and compute.
Meanwhile, despite that humans have very limited working memory capacity \cite{cowan2001magical,cowan2016working}, they can manage to read and write long texts by converting working memories to explicit memories (thus saving space) and retrieving only the needed explicit memories for inference (thus saving compute).
Similarly, by saving explicit memories on drives and doing frequent and constant-size retrieval, LLMs can handle arbitrarily long contexts with time complexity $O(l\log l)$ instead of $\Theta(l^2)$, where $l$ is the context length.

\item Memory consolidation: Instead of writing a piece of knowledge directly into implicit memory, i.e. training model parameters, LLM can first convert it to explicit memory through plain encoding, and then convert this explicit memory to implicit memory through a low-cost step such as compression and finetuning, thus reducing the overall cost.

\item Factuality and interpretability: Encoding texts as explicit memories is less susceptible to information loss compared to dissolving them in model parameters.
With more factual details provided by explicit memories, the LLMs would have less tendency to hallucinate.
Meanwhile, the correspondence of explicit memories to readable texts makes the inference more transparent to humans, and also allows the LLM to consciously examine its own thought process.
\end{itemize}
We demonstrate the improved factuality in the experiments section, and leave the rest to future work.

\vspace{0.5em}
In this work, we introduce a novel architecture and training scheme for LLM based on explicit memory.
The architecture is called \name{}, as explicit memory is the third form of memory in LLM after working memory (context key-values) and implicit memory (model parameters).
\begin{itemize}
\item \name{} utilizes explicit memories during inference, alleviating the burden of model parameters to memorize specific knowledge.
\item The explicit memories are encoded from our knowledge base, and our sparse memory format maintains a realistic storage size.
\item We trained from scratch a \name{} model with 2.4B non-embedding parameters, and its performance surpasses SOTA models with greater sizes.
It also enjoys better performance and faster inference than RAG.
\item Furthermore, \name{} boosts factuality and alleviates hallucination, and it enables fast adaptation to professional tasks.
\end{itemize}
This paper is structured as follows:
Section \ref{sec:theory} lays the theoretical foundation for \name{}, in particular our definitions of knowledge and memory.
Section \ref{sec:design} discusses the basic design of \name{}, including its architecture and training scheme.
Sections \ref{sec:data}, \ref{sec:pretrain}, and \ref{sec:finetune} describes the training of \name{}.
Section \ref{sec:eval} evaluates the performance of \name{} on general benchmarks and professional tasks.
Finally, Section \ref{sec:conclude} concludes this paper and discusses future works.

\subsection{Related work}

\subsubsection{Retrieval-augmented Training} 
Several language models have incorporated text retrieval from the pretraining stage.
REALM \cite{guu2020REALM} augments a BERT model with one retrieval step to solve QA tasks.
Retro \cite{borgeaud2022retro} enhances auto-regressive decoding with multiple rounds of retrieval, once per 64 tokens.
The retrieved texts are injected through a two-layer encoder and then several cross-attention layers in the decoder.
Retro++ \cite{wangRetro++2023} explores the scalability of Retro by reproducing Retro up to 9.5B parameters.

Meanwhile, several models are adapted to retrieval in the finetuning stage.
WebGPT \cite{nakano2021webgpt} learns to use search engine through imitation learning in a text-based web-browsing environment.
Toolformer \cite{schick2024toolformer} performs decoding with multiple tools including search engine, and the finetuning data is labeled by the LM iself.

The closest model to ours is Retro.
Unlike explicit memory, Retro needs to encode the retrieved texts in real-time during inference.
To alleviate the cost of encoding these references, it chooses to use a separate, shallow encoder and also retrieve few references.
Intuitively, this compromise greatly reduces the amount of knowledge that can be extracted and supplied to inference.



Another line of research utilizes retrieval to aid long-context modeling.
Memorizing Transformer \cite{wu2021memorizing} extends the context of language models by an approximate kNN lookup into a non-differentiable cache of past key-value pairs.
LongLlama \cite{tworkowskiFocusedTransformerLongLlama2023} enhances the discernability of context key-value pairs by a finetuning process inspired by contrastive learning.
LONGMEM \cite{wang2024LONGMEM} designs a decoupled architecture to avoid the memory staleness issue when training the Memorizing Transformer.
These methods are not directly applicable to large knowledge bases since the resulting key-value caches will occupy enormous space.
Our method overcomes this difficulty through a more intense memory sparsification method.


\subsubsection{Sparse Computation} 
To combat the aforementioned knowledge traversal problem and improve knowledge efficiency, ongoing works seek novel architectures that process each token with a minimum and adaptive subset of model parameters.
This adaptive sparsity is also known as contextual sparsity \cite{liu2023dejavu}.
The Mixture-of-Experts (MoE) use sparse routing to assign Transformer submodules to tokens, scaling model capacity without large increases in training or inference costs.
The most common MoE design \cite{fedus2022switch} hosts multiple MLP layers in each Transformer block and routes each token to a few MLPs with the highest allocation score predicted by a linear classifier.
Furthermore, variants based on compression such as QMoE \cite{DBLP:journals/corr/abs-2310-16795} are introduced to alleviate the memory burden of MoE.
Despite the sparse routing, the boost in parameter efficiency is usually bounded by $4\sim32$.
For instance, the Arctic model \cite{snowflake2024arctic}, one of the sparsest MoE LLM in recent years, has an active parameter ratio of about $3.5\%$.
Similarly, the Mixture of Depth architecture processes each token with an adaptive subset of the model layers.
The implementations can be based on early exit \cite{elbayad2020depthadaptive} or top-$k$ routing \cite{raposo2024mixtureofdepths}, reducing the amount of compute to $12.5\sim50\%$.
More fine-grained approaches can perform sparsification at the level of individual MLP neurons and attention heads.
The model Deja Vu \cite{liu2023dejavu} trains a low-cost network for each MLP/attention layer that predicts the relevance of each neuron/head at this layer to each token.
Then, during inference, Deja Vu keeps the top $5\sim15\%$ MLP neurons and $20\sim50\%$ attention heads for each token.

\subsubsection{Parameter as memory} 

Several works have portrayed model parameters as implicit memory, in accordance with our philosophy.
\cite{geva2021kv} demonstrates that the neurons in the MLP layers of GPTs behave like key-value pairs.
Specifically, with the MLP layer written as $\sigma(XK^T)V$, each row of the first layer weight $K_i$ functions like a key vector, with the corresponding row in the second layer weight $V_i$ being the value vector.
\cite{geva2021kv} observes that for most of the MLP neurons, the $K_i$ is activated by context texts that obey some human interpretable pattern, and the $V_i$ activates the column of the output matrix that corresponds to the most probable next token of the pattern (e.g. $n$-gram).
Based on this observation, \cite{sukhbaatar2019augmenting} designs a GPT variant that consists of only attention layers, with performance matching that of the usual GPTs.
The MLP layers are incorporated into the attention layers in the form of key-value vector pairs, which are called persistent memories.
Similarly, using sensitivity analysis, \cite{dai2021knowledge} discovers that factual knowledge learned by BERT is often localized at one or few MLP neurons.
These neurons are called ``knowledge neurons", and by manipulating them, \cite{dai2021knowledge} manages to update single pieces of knowledge of BERT.
Meanwhile, \cite{elhage2022superposition} studies an interesting phenomenon known as superposition or polysemanticity, that a neural network can store many unrelated concepts into a single neuron.

\section{Memory Circuitry Theory}
\label{sec:theory}

This section introduces our memory circuitry theory, which defines knowledge and memory in the context of LLM.
We will see that this theory helps to determine which knowledge can be stored as explicit memory, and what kind of model architecture is suitable for reading and writing explicit memories.
For readers interested primarily in the results, it may suffice to review Claim \ref{claim:separable} and Remark \ref{remark:induction_to_architecture} before proceeding to the subsequent sections.
The concepts to be discussed are illustrated in Figure \ref{fig:abstract_specific_knowledge}.

\begin{figure}[h]
\centering
\includegraphics[width=0.66\textwidth]{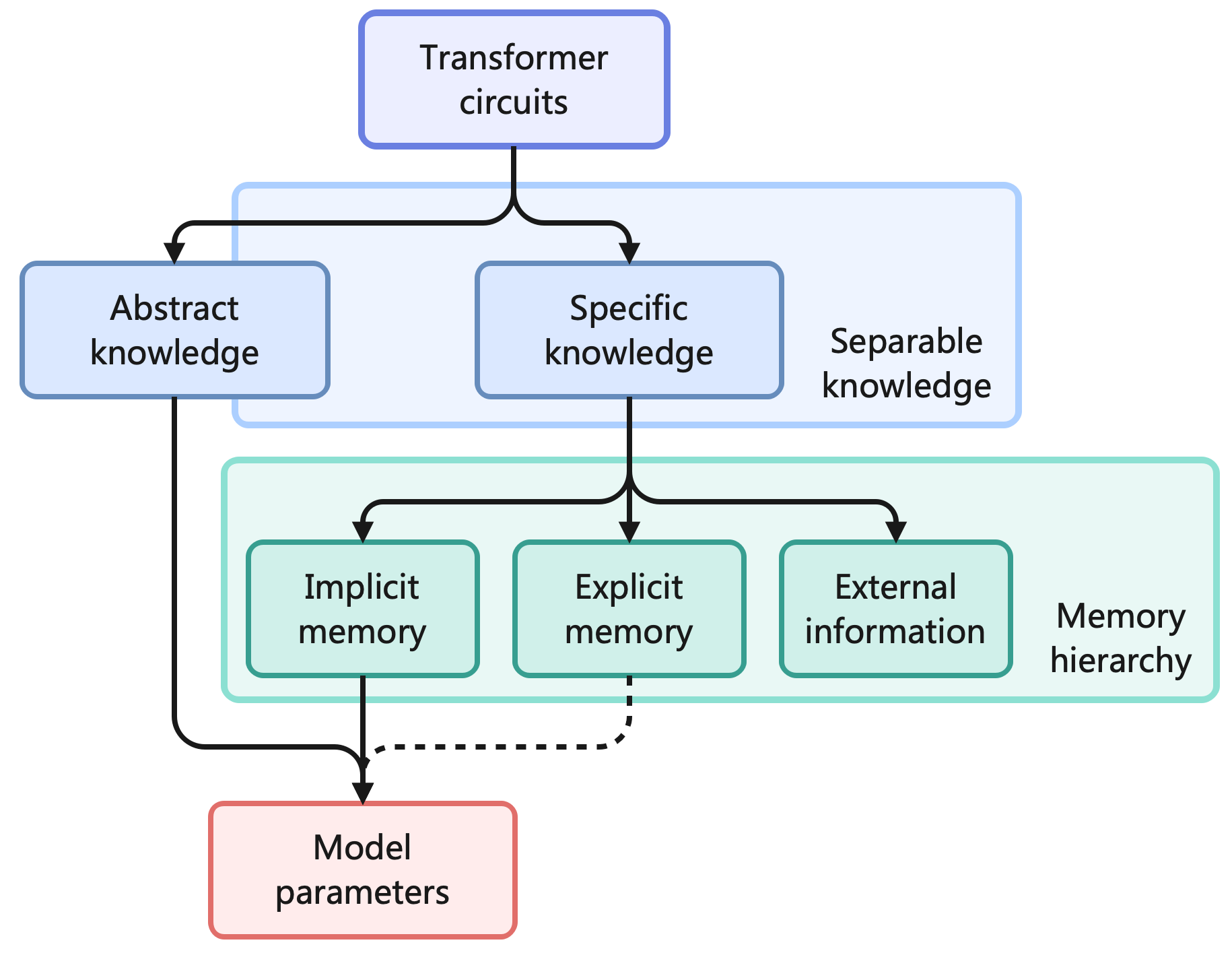}
\caption{Categorization of knowledge and memory formats.
The explicit memories, extracted from model activations, lie half-way between raw data and model parameters, so we use a dotted line to indicate that they may or may not be regarded as parameters.}
\label{fig:abstract_specific_knowledge}
\end{figure}

\subsection{Preliminaries}

The objective is to decompose the computations of a LLM into smaller, recurring parts, and analyze which parts can be separated from the LLM.
These small parts will be defined as the ``knowledge" of the LLM, and this characterization helps to identify what knowledge can be externalized as explicit memory, enabling both the memory hierarchy and a lightweight backbone.

One behaviorist approach is to define the smaller parts as input-output relations between small subsequences, such that if the input text contains a subsequence belonging to some pattern, then the output text of the LLM contains a subsequence that belongs to some corresponding pattern.
\begin{itemize}
\item One specific input-output relation is that if the immediate context contains ``China" and ``capital", then output the token ``Beijing".
\item One abstract input-output relation is that if the immediate context is some arithmetic expression (e.g. ``$123\times456=$") then output the answer (e.g. ``$56088$").
\item One abstract relation that will be mentioned frequently is the ``search, copy and paste" \cite{olsson2022incontext}, such that if the context has the form ``\dots[a][b]\dots[a]" then output ``[b]", where [a] and [b] are arbitrary tokens.
\end{itemize}
A decomposition into these relations seems natural since autoregressive LLMs can be seen as upgraded versions of $n$-grams, with the fixed input/output segments generalized to flexible patterns and with the plain lookup table generalized to multi-step computations.

Nevertheless, a behaviorist approach is insufficient since an input-output relation alone cannot uniquely pin down a piece of knowledge:
a LLM may answer correctly to arithmetic questions based on either the actual knowledge of arithmetic or memorization (hosting a lookup table for all expressions such as ``$123\times456=56088$").
Therefore, we take a white-box approach that includes in the definition the internal computations of the LLM that convert these inputs to the related outputs.

Here are two preliminary examples of internal computations.


\begin{example}
\label{ex:capital-prelim}
Several works have studied the underlying mechanisms when LLMs answer to the prompt ``The capital of China is" with ``Beijing", as well as other factual questions \cite{dai2021knowledge,geva2021kv,lv2024interpreting,chughtai2024summing}.
At least two mechanisms are involved, and the LLM may use their superposition \cite{lv2024interpreting}.
One mechanism is to use general-purpose attention heads (called ``mover heads") to move ``capital" and ``China" to the last token ``is", and then use the MLP layers to map the feature of the last token to ``Beijing" \cite{lv2024interpreting}.
Often, only one or a few MLP neurons are causally relevant, and they are called ``knowledge neurons" \cite{dai2021knowledge}.
This mechanism is illustrated in Figure \ref{fig:classical_circuits} (left).
Another mechanism involves attention heads $h$ whose value-to-output matrices $W^h_V W^h_O$ function like bigrams,
e.g. mapping ``captial" to \{``Paris", ``Beijing", \dots\} and ``China" to \{``panda", ``Beijing", \dots\} , which sum up to produce ``Beijing" \cite{chughtai2024summing,geva2021kv,lv2024interpreting}.
This mechanism is illustrated in Figure \ref{fig:classical_circuits} (middle).
\end{example}

\begin{example}
\label{ex:induction-prelim}
The ability of LLMs to perform ``search, copy and paste", namely answering to the context ``\dots[a][b]\dots[a]" with ``[b]", is based on two attention heads, together called induction heads \cite{olsson2022incontext}.
The first head copies the feature of the previous token, enabling [b] to ``dress like" its previous token [a].
The second head searches for similar features, enabling the second [a] to attend to [b], which now has the appearance of [a].
Thereby, the last token [a] manages to retrieve the feature of [b] and to output [b].
This mechanism is illustrated in Figure \ref{fig:classical_circuits} (right).
A similar mechanism is found for in-context learning \cite{wang2023label}.
\end{example}

\begin{figure}[h]
\centering
\includegraphics[width=0.85\textwidth]{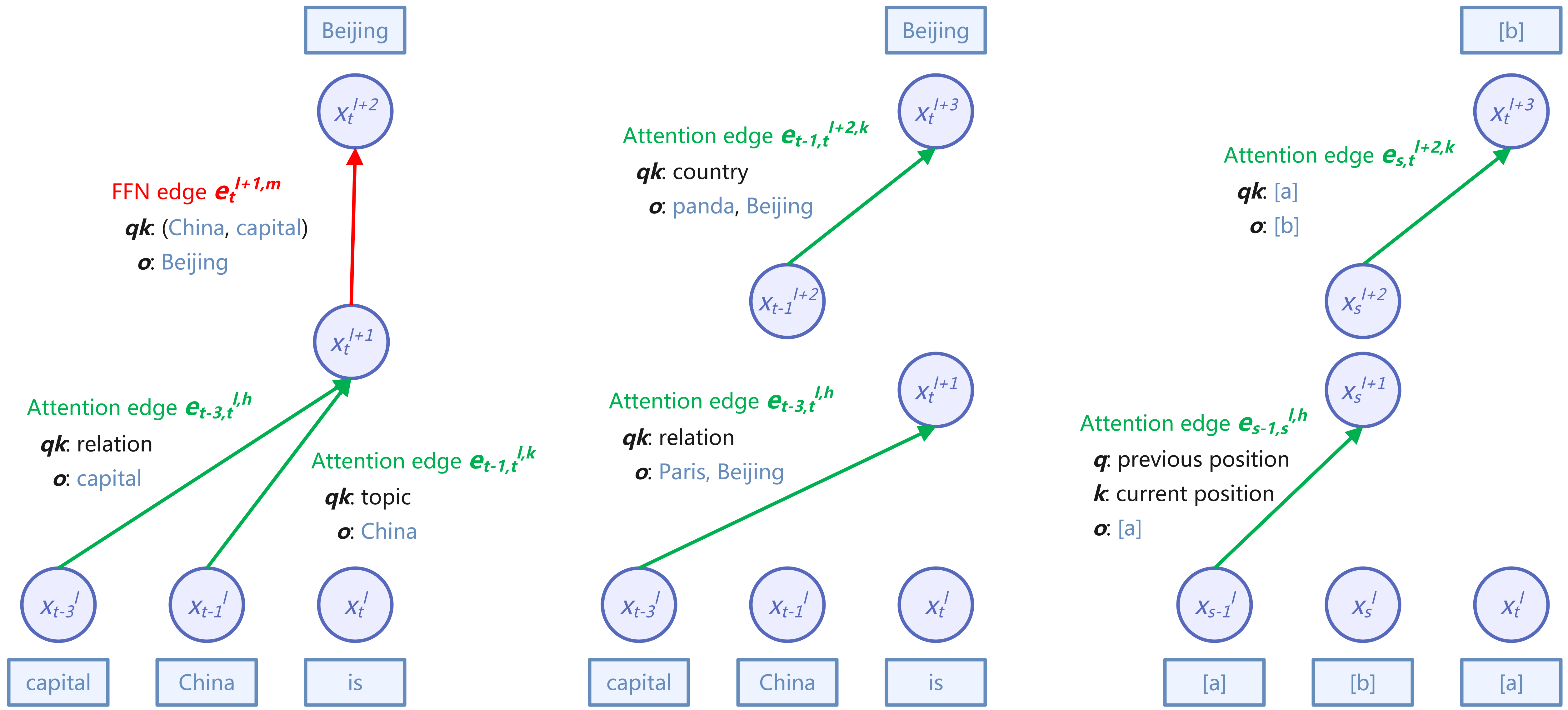}
\caption{Illustration of three subgraphs.
Left: A subgraph that inputs ``the capital of China is" and outputs ``Beijing".
The knowledge neuron is marked in red and the mover heads in green.
Middle: Another subgraph with similar function using task-specific heads.
Right: The induction-heads subgraph that inputs ``[a][b]...[a]" and outputs [b], where [a], [b] are arbitrary tokens.
The notations are introduced in Section \ref{sec:knowledge}.
The locations of these attention heads and MLP neurons may be variable.}
\label{fig:classical_circuits}
\end{figure}

We will address the internal mechanism for an input-output relation as a circuit, and will define a piece of knowledge as an input-output relation plus its circuit.
By manipulating these circuits, one can separate many pieces of knowledge from a LLM while keeping its function intact.

Recent works on circuit discovery demonstrate that some knowledge and skills possessed by Transformer LLMs can be identified with patterns in their computation graphs \cite{olsson2022incontext,wang2023label,stolfo2023mechanistic,geva2023dissecting,wang2022interpretability,conmy2023towards,dai2021knowledge,geva2021kv}, but there has not been a universally accepted definition of circuit.
Different from works on Boolean circuits \cite{hao2022formal,merrill2023logic} and circuits with Transformer submodules as their nodes \cite{conmy2023towards,yao2024knowledge}, we characterize a circuit as a ``spatial-temporal" phenomenon, whose causal structure is localized at the right places (MLP neurons and attention heads) and right times (tokens).
Thus, we define a computation graph as a directed acyclic graph, whose nodes are the hidden features of all tokens at all all MLP and attention layers, and whose edges correspond to all activations inside these layers.
In particular, the computation graph hosts one copy of the Transformer architecture at each time step.
To transcend this phenomenological characterization, we define a circuit as an equivalence class of similar subgraphs across multiple computation graphs.

As a remark, it is conceptually feasible to identify a circuit with the minimal subset of Transformer parameters that causes this circuit.
The benefit is that such definition of knowledge seems more intrinsic to the LLM.
Nevertheless, with the current definition, it is easier to perform surgery on the circuits and derive constructive proofs.
Besides, it is known that Transformer submodules exhibit superposition or polysemanticity, such that one MLP neuron or attention head may serve multiple distinct functions \cite{elhage2022superposition,lv2024interpreting}, making the identification of parameter subsets a challenge task.

\subsection{Knowledge}
\label{sec:knowledge}

We begin with the definition of the knowledge of LLMs.
For now, it suffices to adopt heuristic definitions instead of fully rigorous ones.
Throughout this section, by LLM we mean autoregressive Transformer LLM that has at least been pretrained.
Let $L$ be the number of Transformer blocks and $H$ be the number of attention heads at each attention layer, and the blocks and heads are numbered by $l=0, \dots L-1$ and $h=0,\dots H-1$.
There are in total $2L$ layers (MLP layers and attention layers), and the input features to these layers are numbered by $0, \dots 2L-1$.

\begin{definition}
Given an LLM and a text $\mathbf{t}=(t_0,\dots t_n)$, the \textbf{computation graph} $G$ on input $(t_0, \dots t_{n-1})$ and target $(t_1, \dots t_n)$ is a directed graph with weighted edges such that
\begin{itemize}
\item Its nodes consist of the hidden vectors $\x_i^{2l}$ before all attention layers, the hidden vectors $\x_i^{2l+1}$ before all MLP layers, and the output vectors $\x_i^{2L}$, for all blocks $l=0, \dots L-1$ and positions $i=0, \dots n-1$.
\item Its directed edges consist of each attention edge $e^{l,h}_{i,j}$ that goes from $\x_i^{2l}$ to $\x_j^{2l+1}$ at the $h$-th head of the $l$-th attention layer for all $l,h$ and $i\leq j$, as well as each MLP edge $e^{l,m}_i$ that goes from $\x_i^{2l+1}$ to $\x_i^{2l+2}$ through the $m$-th neuron of the $l$-th MLP layer for all $l,m,i$.

\item The weight of each attention edge $e^{l,h}_{i,j}$, which measures the influence of the attention score $a^{l,h}_{i,j}$ on the LLM output, is defined by
\begin{equation*}
\mathcal{L}-\mathcal{L}\big|_{a^{l,h}_{i,j}=0} \quad \text{or} \quad \frac{\partial \mathcal{L}}{\partial a^{l,h}_{i,j}}
\end{equation*}
where $\mathcal{L}$ is the log-likelihood of the target $(t_1, \dots t_n)$, with $\mathcal{L}|_{a=0}$ obtained by setting $a=0$ (i.e. causal intervention).
Similarly, the weight of each MLP edge $e^{l,m}_i$, which measures the influence of the neuron activation $a^{l,m}_i$ on the LLM output, is defined likewise.

\item Given any subgraph $S \subseteq G$, define the \textbf{associated input} of $S$ as a subsequence $\mathbf{t}_{\text{in}}(S) \subseteq (t_0, \dots t_{n-1})$ such that a token $t_i$ belongs to $\mathbf{t}_{\text{in}}(S)$ if and only if $\big\|\nabla_{\x_i^0} a\big\|$ is large for some attention edge (or MLP edge) in $S$ with attention score (or activation) $a$.

\item Similarly, define the \textbf{associated output} of the subgraph $S$ as a subsequence $\mathbf{t}_{\text{out}}(S) \subseteq (t_1, \dots t_n)$ such that a token $t_i$ belongs to $\mathbf{t}_{\text{out}}(S)$ if and only if
\begin{equation*}
\mathcal{L}_i-\mathcal{L}_i\big|_{a=0} \quad \text{or} \quad \frac{\partial \mathcal{L}_i}{\partial a}
\end{equation*}
is large for some attention edge (or MLP edge) in $S$ with attention score (or activation) $a$.
Here $\mathcal{L}_i$ is the log-likelihood of $t_i$ with respect to the LLM output.
\end{itemize}
\end{definition}

\begin{figure}[h]
\centering
\subfloat{\includegraphics[width=0.255\textwidth]{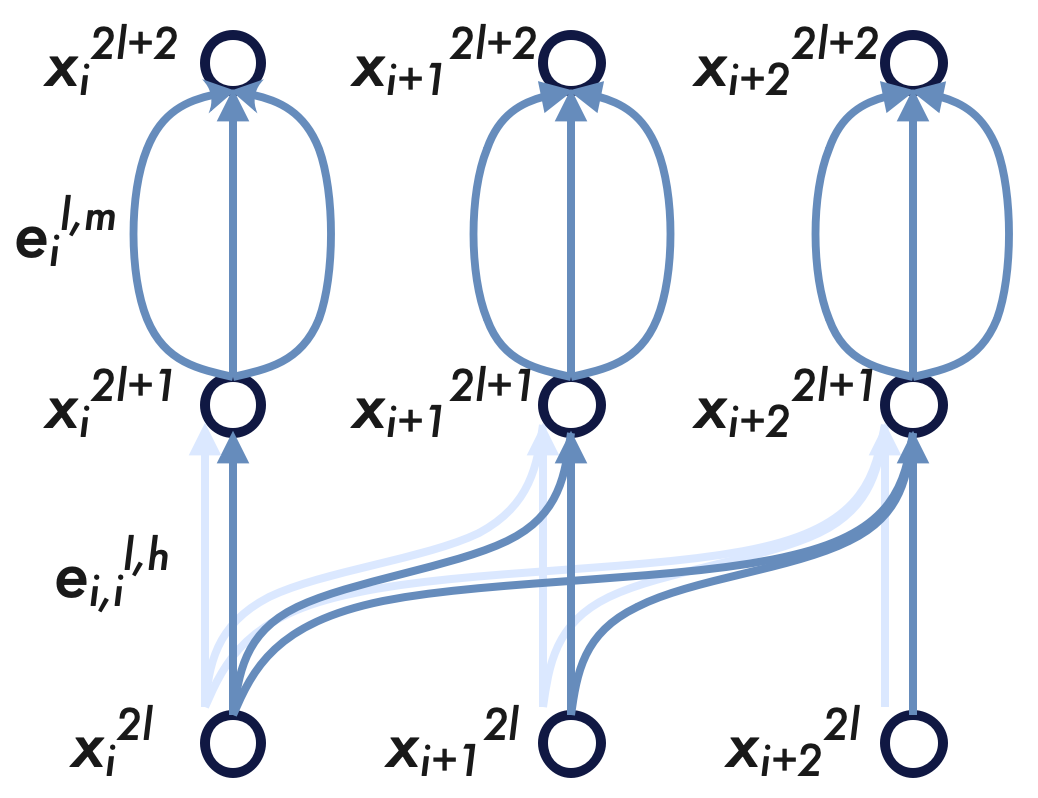}}
\quad
\subfloat{\includegraphics[width=0.72\textwidth]{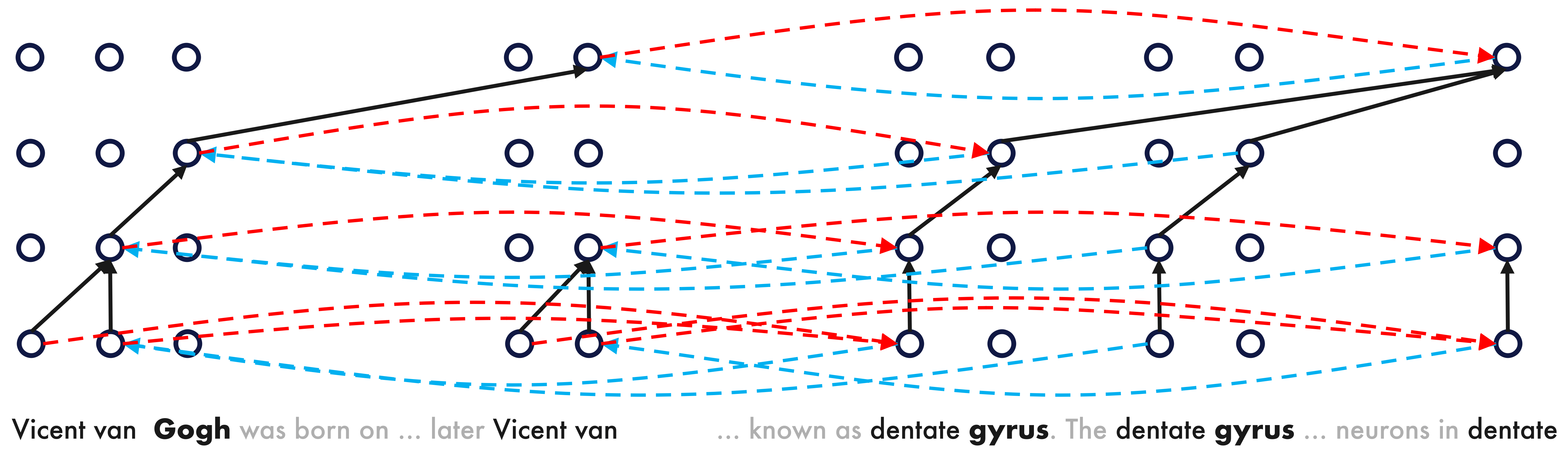}}
\caption{Left: Illustration of the computation graph over one Transformer block, showing only three tokens, one attention head and three MLP neurons.
The edge weights are not shown.
Right: The subgraphs $S_1, S_2$, namely the induced subgraphs of the attention edges (black arrows), belong to the circuit of the induction head.
The red arrows denote a homomorphism from $S_1$ to $S_2$, and the blue arrows denote a homomorphism from $S_2$ to $S_1$.}
\label{fig:circuit}
\end{figure}

\begin{definition}
Given two computation graphs $G_1, G_2$ of an LLM and their subgraphs $S_1, S_2$, a mapping $f$ from the nodes of $S_1$ to the nodes of $S_2$ (not necessarily injective) is a \textbf{homomorphism} if
\begin{itemize}
\item every node at depth $l \in \{0, \dots 2L\}$ is mapped to depth $l$,
\item if two nodes are on the same position $i$, then they are mapped onto the same position,
\item if two nodes share an edge on attention head $h$ or MLP neuron $m$, then their images also share an edge on head $h$ or neuron $m$.
\end{itemize}
If such a homomorphism exists, then we say that $S_1$ is homomorphic to $S_2$.
\end{definition}

It may be more convenient to define the mapping to be between the input tokens of two sentences, but we adopt the current formulation as it is applicable to more general settings without an obvious correspondence between the tokens and the hidden features at each layer.

\begin{definition}
\label{def:circuit}
Given an LLM and a distribution of texts, a \textbf{circuit} is an equivalence class $\K$ of subgraphs from computation graphs on random texts, such that
\begin{itemize}
\item The computation graph on a random text contains some subgraph $S\in\K$ with positive probability.
\item All subgraphs $S\in\K$ are homomorphic to each other.
\item All edges of all $S\in\K$ have non-negligible weights.
\item The pairs $(\mathbf{t}_{\text{in}}(S), \mathbf{t}_{\text{out}}(S))$ share some interpretable meaning across all $S \in \K$.
\end{itemize}
\end{definition}


\begin{definition}
\label{def:specific}
Given an LLM and a distribution of texts, we call each circuit a \textbf{knowledge}.
Furthermore, a circuit $\K$ is called a
\begin{itemize}
\item \textbf{specific knowledge}, if the associated inputs $t_{\text{in}}(S)$ for all subgraphs $S \in\K$ share some interpretable meaning, and the associated outputs $t_{\text{out}}(S)$ for all $S\in\K$ are the same or differ by at most a small fraction of tokens.
\item \textbf{abstract knowledge}, else.
\end{itemize}
\end{definition}

From now on, we use knowledge as a countable noun since the circuits are countable.
Note that the criterion in Definition \ref{def:specific} is stronger than the last criterion in Definition \ref{def:circuit}, e.g. consider the circuit that always copy-and-pastes the previous token.
We will see that the rigidity of specific knowledges makes them easier to externalize.

Here are some well-known examples of knowledge.

\begin{example}
Recall the knowledge neuron from Example \ref{ex:capital-prelim} that helps to answer ``The capital of China is Beijing".
Such neurons can be activated by a variety of contexts that involve the subject-relation pair (``China", ``capital") \cite{dai2021knowledge}.
Its circuit can be simply defined as the equivalence class of subgraphs induced by edges $e^{l,m}_i$, where $(l,m)$ is the fixed location of the knowledge neuron and $i$ is the variable position of the last token of the context.
The associated inputs are ``China" and ``capital", and the associated outputs are always ``Beijing".
By definition, this circuit is a specific knowledge, since its associated output is fixed and its associated inputs share a clear pattern (fixed tokens with variable positions).
\end{example}
Similarly, by straightforward construction, one can show that each $n$-gram can be expressed as a specific knowledge.

\begin{example}
\label{ex:induction}
Recall the induction heads \cite{olsson2022incontext} from Example \ref{ex:induction-prelim} that complete ``[a][b] \dots [a]" with ``[b]".
Let $(l,h), (l+1,h')$ be the locations of these two heads, and denote the variable positions of the two token [a]'s by $i,j$.
Its circuit is the equivalence class of subgraphs induced by the two edges $e^{l,h}_{i,i+1}, e^{l+1,h'}_{i+1,j}$.
Although the associated input-output pairs ``[a][b]\dots[a][b]" have a clear pattern, the associated outputs ``[b]" alone can be arbitrary, so the induction head is an abstract knowledge.
\end{example}
More sophisticated abstract knowledges have been identified for in-context learning \cite{wang2023label} and indirect object identification \cite{wang2022interpretability}.

\begin{definition}
\label{def:realization}
Given a LLM and a knowledge $\K$, a text $\mathbf{t}=(t_0, \dots t_n)$ is called a \textbf{realization} of $\K$, if the computation graph on $\mathbf{t}$ has a subgraph that belongs to $\K$.
\end{definition}

For instance, any text of the form [a][b]\dots[a][b] can be a realization of the abstract knowledge of induction head.

Our definition of knowledge is extrinsic, depending on a specific LLM, instead of intrinsic, depending only on texts.
From this perspective, Problem (\ref{eq:cost}) can be interpreted as relocating the knowledges from an all-encompassing LLM to more efficient models equipped with memory hierarchy.
For concreteness, one can fix this reference LLM to be the latest version of ChatGPT or Claude \cite{openai2023gpt4,anthropic2024claude}, or some infinitely large model from a properly defined limit that has learned from infinite data.

\begin{assumption}[Completeness]
\label{assume:complete}
Fix a reference LLM and a distribution of texts, let $G$ be the computation graph of a random text.
Assume that there exists a set $\mathfrak{K}$ of knowledges such that, with probability 1 over the random text,
the subgraph of $G$ induced by edges with non-negligible weights can be expressed as a union of subgraphs $\{S_i \in \K_i\}$ from $\{\K_i\}\subseteq\mathfrak{K}$.
\end{assumption}

Essentially, Assumption \ref{assume:complete} posits that all computations in the LLM can be fully decomposed into circuits, so that the LLM is nothing more than a collection of specific and abstract knowledges.
This viewpoint underscores that the efficiency of LLMs is ultimately about the effective organization of these knowledges, an objective partially addressed by Problem (1).

\subsection{Memory}
\label{sec:memory}

Now the question is what knowledge can be separated from the model parameters and moved to the lower levels of the memory hierarchy.

\begin{definition}
\label{def:separable}
A knowledge $\K$ of the reference LLM is \textbf{separable} if there exists another LLM $M$ such that
\begin{itemize}
\item $M$ does not possess this knowledge, such that for any realization $\mathbf{t}$ of $\K$, the model $M$ cannot generate each token of the associated output $\mathbf{t}_{\text{out}}$ with high probability, e.g. $\prob_M(t_i|t_0\dots t_{i-1}) \leq 1/2$ for some $t_i \in \mathbf{t}_{\text{out}}$.
\item There exists a text $\mathbf{t}_*$ such that for any realization $\mathbf{t}$ of $\K$,
the model $M$ using $\mathbf{t}_*$ as prefix can generate each token of the associated output $\mathbf{t}_{\text{out}}$ with high probability,
e.g. $\prob_M(t_i|\mathbf{t}_* t_0\dots t_{i-1}) \geq 0.9$ for every $t_i \in \mathbf{t}_{\text{out}}$.
\end{itemize}
If among the realizations of $\K$, the same associated input $\mathbf{t}_{\text{in}}$ can correspond to multiple associated outputs $\mathbf{t}_{\text{out}}$, then the above probabilities are summed over all branches if position $i$ is a branching point.
\end{definition}

\begin{definition}
\label{def:imitable}
A separable knowledge $\K$ of the reference LLM is \textbf{imitable} if any realization $\mathbf{t}'$ of $\K$ can be used as the prefix $\mathbf{t}_*$ in Definition \ref{def:separable},
e.g. for any realizations $\mathbf{t},\mathbf{t}'$ of $\K$, we have $\prob_M(t_i|\mathbf{t}' t_0\dots t_{i-1}) \geq 0.9$ for every $t_i \in \mathbf{t}_{\text{out}}$.
\end{definition}

Basically, imitability means that LLMs can achieve the same effect as possessing this knowledge by retrieving example texts that demonstrate this knowledge.
Few-shot prompting can be seen as a special case of providing realizations.

Separability is a more general property than imitability.
For instance, one can set the prefix $\mathbf{t}_*$ to be an abstract description of $\K$ instead of its realization, and this is reminiscent of instruction prompting.
Nevertheless, it is not obvious whether the set of separable knowledges is strictly larger than the set of imitable knowledges.

\begin{claim}
\label{claim:separable}
Every specific knowledge $\K$ is imitable and thus is separable.
\end{claim}
\begin{proof}[Proof (informal)]
Without loss of generality, we can assume that for any realization $\mathbf{t}$ of $\K$, all tokens of the associated input $\mathbf{t}_{\text{in}}$ precede all tokens of the associated output $\mathbf{t}_{\text{out}}$.
Otherwise, we can split $\mathbf{t}_{\text{in}}$ into two halves $\mathbf{t}_1, \mathbf{t}_2$ that precedes/does not precede $\mathbf{t}_{\text{out}}$, and split the corresponding subgraph $S\in\K$ into two halves $S_1,S_2$ that have high weights with respect to $\mathbf{t}_1, \mathbf{t}_2$.
Using monotonicity arguments once Definition \ref{def:circuit} is fully formalized, one can try to show that this splitting is invariant across $S\in\K$ and therefore the sets of $S_1, S_2$ are two specific knowledges.

Consider sequences of the form [a][b]\dots[a'][b'], where [a], [a'] (or [b], [b']) could be the associated inputs (or outputs) of any subgraphs $S,S'\in\K$.
By Definition \ref{def:specific}, [a] and [a'] always share some interpretable meaning, while [b] and [b'] are approximately the same sequence.
One can construct an abstract knowledge that completes [a][b]\dots[a'] with [b']:
the first part of this circuit detects the common feature of the [a]'s (possibly overlapping with the subgraphs of $\K$), the second part is an induction head (analogous to Example \ref{ex:induction}, it provides [b] with the common feature of the [a]'s and lets [a'] to attend to [b]), and the third part generates [b'] based on [b] with possible slight modifications.
This circuit is an abstract knowledge since it can be applied to other specific knowledges as long as their associated inputs share the same meaning with the [a]'s, no matter how their associated outputs could vary.

Meanwhile, construct the model $M$ by letting the reference model forget $\K$ (e.g. by finetuning on a modified data distribution such that the associated input of $\K$ is never followed by the associated output, while the rest of the distribution remains the same).
Combining this circuit with $M$ completes the proof.
\end{proof}

Claim \ref{claim:separable} indicates that a lot of knowledges can be externalized from the model parameters.
The converse of Claim \ref{claim:separable} may not hold, since it is imaginable that some abstract knowledges can also be substituted with their realizations.

\begin{remark}
\label{remark:induction_to_architecture}
There are three details in the proof of Claim \ref{claim:separable} that will be useful later
\begin{enumerate}
\item The circuit we construct has only one attention head that attends to the reference text $\mathbf{t}'$ from the present text $\mathbf{t}$, while all other computations are confined within either $\mathbf{t}$ or $\mathbf{t'}$.
\item Moreover, in this attention head, the circuit only needs the edges from [b] to [a'].
Thus, in general this head only needs to attend to very few tokens in the reference.
\item It suffices for the reference $\mathbf{t}'$ to attend only to itself.
\end{enumerate}
These properties will guide our architecture design.
\end{remark}

To finish the set-up of Problem (\ref{eq:cost}), we define the memory formats.
The definition should subsume the aforementioned formats of model parameters, explicit memories and plain texts for RAG, and also allow for new memory formats of future LLMs.

\begin{definition}
Let $\mathfrak{K}$ be the complete set of knowledges from Assumption \ref{assume:complete} and consider the subset of separable knowledges.
Let $\mathfrak{T}$ be a set that contains one or several realizations $\mathbf{t}$ for each separable knowledge.
Let $f_1, \dots f_m$ be any functions over $\mathfrak{T}$.
Abstractly speaking, a memory-augmented LLM $M$ is some mapping from prefixes to token distributions with additional inputs
\begin{equation}
\label{eq:read}
M: \big( (t_0 \dots t_{i-1}), \{\K_1, \dots \K_N\}, X_1, \dots X_m \big) \mapsto \prob(\cdot|t_0 \dots t_{i-1})
\end{equation}
where the set $\{\K_1, \dots \K_N\}$ consists of non-separable knowledges of $M$ that are invoked at this step, and the sets $X_j$ consist of encoded texts
\begin{equation}
\label{eq:write}
X_j = \big\{ f_j(\mathbf{t}_{j,k}) \big\}
\end{equation}
for some $\mathbf{t}_{j,k} \in \mathfrak{T}$.

Each $j=1,\dots m$ represents a memory format and $f_j$ is called the \texttt{write} function of this format.
If some realization of a separable knowledge $\K$ participates in the mapping $M$, then we say that $\K$ is \texttt{written} in format $j$ and \texttt{read} by $M$.
\end{definition}

Analogous to Assumption \ref{assume:complete}, we are decomposing each step of LLM inference into the invoked circuits, but the decomposition here also involves reference texts that are written in various memory formats.

Table \ref{table:memory} demonstrates that the write functions could be diverse, and the list is probably far from conclusive.
Nevertheless, some heuristics still apply.
The write function $f_j$ and the read process in $M$ for each format $j$ should be non-trivial such that, for any separable knowledge $\K$ not contained in $M$ and any realization $\mathbf{t}$ of $\K$, if $\K$ enters in $M$ through format $j$, then $M$ should be able to generate each token of the associated output of $\K$ in $\mathbf{t}$ with higher probability as in Definition \ref{def:separable}.
Thus, informally speaking, the total cost of writing and reading $\K$ must be bounded from 0, since some minimum computation is necessary for reducing the uncertainty in generating the correct tokens.
It follows that the write cost and read cost are complementary, i.e. cheaper writing must be accompanied by more expensive reading.

We define this inverse relationship between the write cost and read cost as the memory hierarchy.
This relationship is in accordance with our experience regarding the three examples of human memories in Table \ref{table:memory}, e.g. we can utter the common expressions almost immediately while it may take a few seconds to recall a book we read, but the former skill is acquired through years of language speaking.
For the LLM memories in Table \ref{table:memory}, the inverse relationship is illustrated Figure \ref{fig:total_cost_usage_2B} and established by the calculations in Appendix \ref{appendix:cost}.

\begin{figure}[h]
\centering
\includegraphics[width=0.75\textwidth]{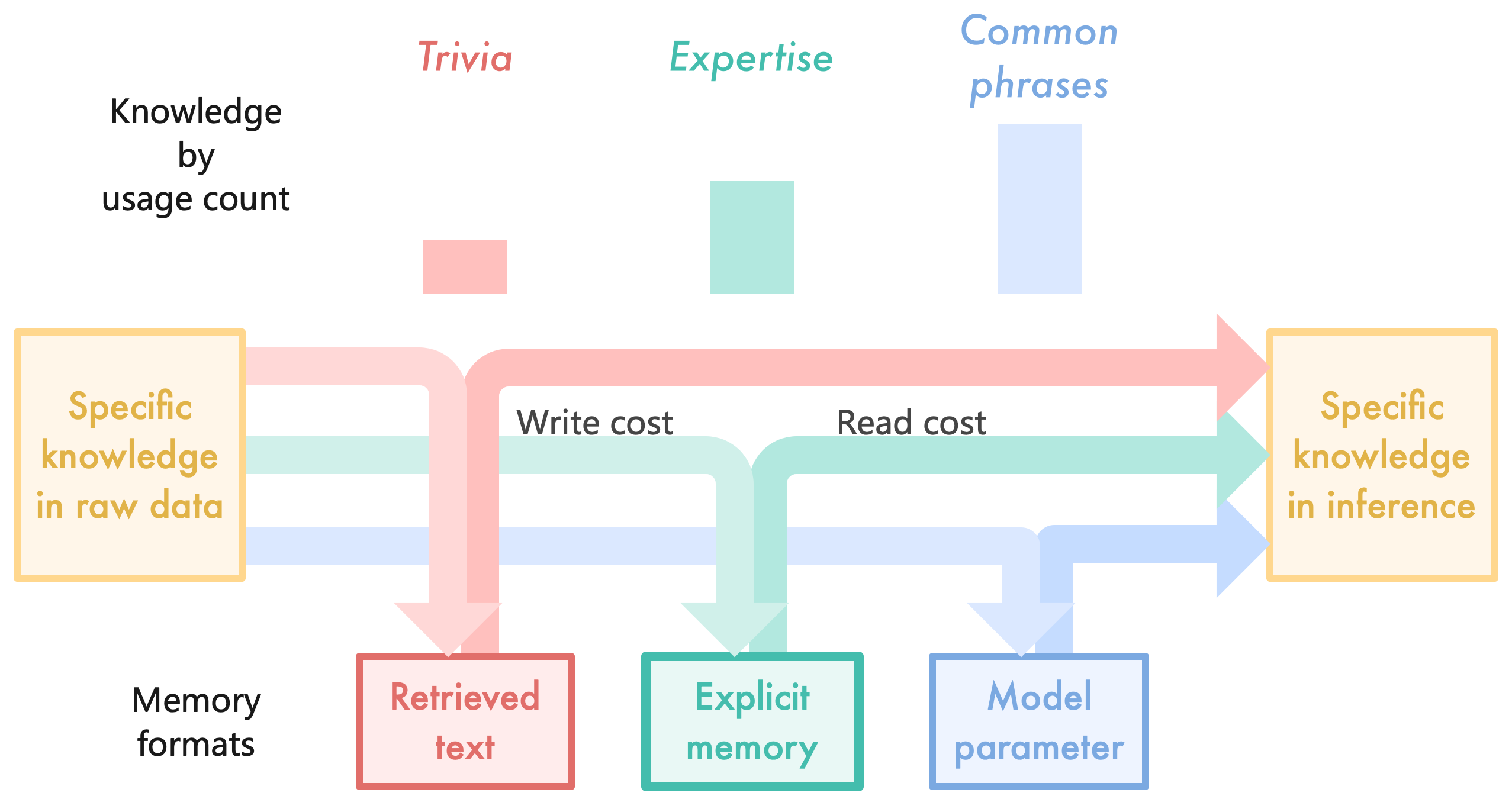}
\caption{Different memory formats with different balances of write cost and read cost.
The specific knowledges with high to low usage counts are exemplified by common expressions, expertise and trivia, and are assigned to implicit memory, explicit memory and external information.}
\label{fig:knowledge_distribution}
\end{figure}

The imbalanced use of knowledges leads to a heterogeneous distribution of knowledges across the memory hierarchy.
To minimize the total cost (\ref{eq:cost}), the separable knowledges that are used more often should be assigned to memory formats with high write cost and low read cost, whereas the rarely used knowledges should be assigned to formats with low write cost and high read cost.
Also, adding a new memory format $m+1$ is always beneficial as it expands the search space and decreases the minimum cost whenever the usage count of some knowledge $\K$ lies in the interval
\begin{equation*}
[n^-_{m+1}, n^+_{m+1}] = \big\{ n \in [0,\infty) ~\big|~ \text{argmin}_j ~\text{cost}_{\text{write}}(\K, j) + n \cdot \text{cost}_{\text{read}}(\K, j) = m+1 \big\}
\end{equation*}
Examples of these intervals are displayed in Figure \ref{fig:total_cost_usage_2B}.
For concreteness, Figure \ref{fig:knowledge_distribution} depicts a reasonable distribution of the specific knowledges for humans, and we expect a similar distribution to hold for LLMs equipped with explicit memory.


\section{Design}
\label{sec:design}

This section describes the architecture and training scheme of \name{}.

Regarding architecture, the goal is to design an explicit memory mechanism for Transformer LLMs with moderately low write cost and read cost.
In addition, we want to limit the modification to the Transformer architecture to be as little as possible, adding no new trainable parameters, so that most of the existing Transformer LLMs can be converted to \name{} models with little finetuning.
Thus, we arrive at a simple design:
\begin{itemize}
\item Write cost:
Before inference, the LLM writes each reference to an explicit memory, saved on drives.
The memory is selected from the key-value vectors of the self-attention layers, so the write process involves no training.
Each reference is processed independently, avoiding the cost of long-context attention.
\item Read cost:
During inference, explicit memories are retrieved from drives and read by self-attention alongside the usual context key-values.
Each memory consists of very few key-values from a small amount of attention heads, thus greatly reducing the extra compute, GPU storage, drive storage and loading time.
It allows the LLM to retrieve many references frequently with limited influence on decoding speed.
\end{itemize}

Regarding training, the goal is to reduce the cost of pretraining with a more efficient distribution of knowledge.
Based on the discussion in Section \ref{sec:memory}, we want to encourage the LLM to learn only abstract knowledges, with the specific knowledges mostly externalized to the explicit memory bank.
Ideally, the pretraining cost should be reduced to be proportional to the small amount of knowledge stored in the model parameters, thereby taking a step closer to the learning efficiency of humans.

\subsection{Inference Process}
\label{sec:inference}

From now on, we refer to the realizations of separable knowledges (Definitions \ref{def:realization} and \ref{def:separable}) as references.
Our knowledge base (or reference dataset) consists of $1.1 \times 10^8$ text chunks with length bounded by 128 tokens.
Its composition is described in Section \ref{sec:ref}.

Each reference can be converted to an explicit memory, which is a tensor with shape
\begin{center}
(memory layers, 2, key-value heads, sparse tokens, head dimension) = $(22, 2, 8, 8, 80)$
\end{center}
The 2 stands for the key and value, while the other numbers are introduced later.

Before inference, the \name{} model converts all references to explicit memories and save them on drives or non-volatile storage devices.
Then, at inference time, whenever (the id of) a reference is retrieved, its explicit memory is loaded from drives and sent to GPU to be integrated into the computation of \name{}.
By Remark \ref{remark:induction_to_architecture}, a reference during encoding does not need to attend to any other texts (e.g. other references or query texts), so it is fine to encode each reference independently prior to inference.
Such isolation also helps to reduce the compute of attention.

One can also employ a ``cold start" approach to bypass preparation time: each reference is converted to explicit memory upon its initial retrieval, rather than prior to inference.
Subsequent retrievals will then access this stored memory.
The aforementioned inference with precomputed explicit memories will be called ``warm start".

\begin{figure}[h]
\centering
\includegraphics[scale=0.3]{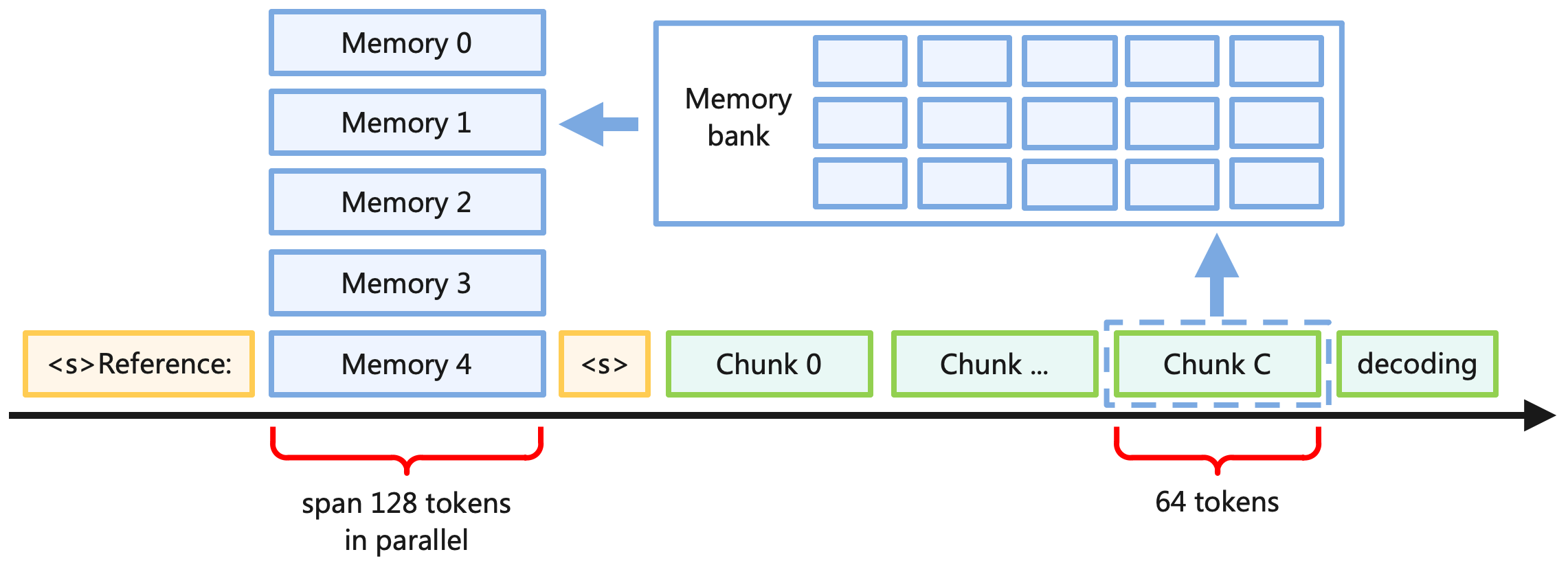}
\caption{The decoding process of \name{} with memory recall.
Each chunk is a fixed-length interval of tokens, which may belong to either the prompt or generated text.}
\label{fig:inference}
\end{figure}

During inference, as illustrated in Figure \ref{fig:inference}, whenever the LLM generates 64 tokens, it discards the current memories, uses these 64 tokens as query text to retrieve 5 new memories, and continues decoding with these memories.
Similarly, when processing the prompt, the LLM retrieves 5 memories for each chunk of 64 tokens.
Each chunk attends to its own memories, and the memories could be different across chunks.
We leave it to future work to optimize these hyperparameters.

The retrieval is performed with plain vector search with cosine similarity.
The references as well as the query chunks are embedded by BGE-M3, a multilingual BERT model \cite{chen2024bge}.
The query and key vectors for retrieval are both obtained from the output feature of the $\langle\text{cls}\rangle$ token.
The vector index is built with FAISS \cite{douze2024faiss}.

To further save time, we maintain a fixed-size cache in RAM to store the most recently used explicit memories.
It's been observed that adjacent chunks often retrieve some of the same references.
So the cache reduces the cost of loading explicit memories from drives.

\begin{remark}
\label{remark:self-retrieval}
It would be ideal to perform retrieval using the hidden features from the LLM itself,
since conceptually the LLM should know its needs better than any external module,
and such internalized retrieval appears more anthropomorphic.
Moreover, retrieving with the hidden features from different layers, different heads and different keywords can help to obtain more diverse results.
One simple implementation is to use the sparsified attention queries of the query text to directly search for the explicit memories.
Since the explicit memories are the attention key-values, such retrieval can work without the need to finetune the LLM.
Specifically, this multi-vector retrieval can follow the routine of \cite{khattab2020colbert} with the additional constraint that a query from attention head $h$ can only search for keys from $h$,
while the sparse attention queries can be obtain using the same selection mechanism for explicit memories described later.
\end{remark}

\begin{remark}
\label{remark:context}
One shortcoming of RAG is that the references are usually text chunks instead of whole documents, and thus during inference the references are encoded without their contexts, making them less comprehensible.
This shortcoming can be easily overcome for explicit memories.
One solution is to encode each document as one sequence, then chunk the attention key-values into 128-token chunks and sparsify them into explicit memories.
This procedure allows the key-values to attend to all their contexts.
\end{remark}

\subsection{Writing and Reading Memory}

Each explicit memory is a subset of the attention key-values from a subset of attention heads when encoding a reference.
Thus, during inference, the LLM can directly read the retrieved explicit memories through its self-attention layers by concatenating them with the usual context key-values (Figure \ref{fig:inference}).
Specially, for each attention head $h$ at layer $l$, if it is chosen as a memory head, then its output $Y^{l,h}$ changes from the usual
\begin{equation*}
Y^{l,h}_i = \text{softmax}\Big(\frac{ X^{l,h}_i W^{l,h}_Q \big(X^{l,h}_{[:i]} W^{l,h}_K\big)^T}{\sqrt{d_h}}\Big) X^{l,h}_{[:i]} W^{l,h}_V W^{l,h}_O
\end{equation*}
where $X_{[:i]}$ denotes all tokens before or at position $i$ and $d_h$ denotes the head dimension, to
\begin{equation}
\label{eq:memory attention}
Y^{l,h}_i = \text{softmax}\Big(\frac{ X^{l,h}_i W^{l,h}_Q \cdot \text{concat}\big(K^{l,h}_0, \dots K^{l,h}_4, X^{l,h}_{[:i]} W^{l,h}_K\big)^T}{\sqrt{d_h}}\Big) \text{concat}\big( V^{l,h}_0, \dots V^{l,h}_4, X^{l,h}_{[:i]} W^{l,h}_V \big) W^{l,h}_O
\end{equation}
where each $(K_j,V_j)$ denotes the keys and values of an explicit memory.

While the context BOS token is $\langle\text{s}\rangle$ as usual,
when encoding each reference we modify the BOS to ``$\langle\text{s}\rangle$Reference:" to help the LLM distinguish between encoding normal texts and encoding references.
This modified BOS is also prepended to the context during inference, as illustrated in Figure \ref{fig:inference}, while the context BOS token now serves as a separator between the references and context.
Unlike the explicit memories which only appear at a subset of attention heads, this modified BOS is placed at every head at every layer.
The motivation is that since the context BOS can attend to the references, its feature is no longer constant, so the LLM needs the modified BOS to serve as the new constant for all attention heads.

Furthermore, we adopt parallel position encoding for all explicit memories, namely the positions of all their keys lie in the same interval of length 128, as depicted in Figure \ref{fig:inference}.
We use the rotary position encoding (RoPE) \cite{su2024roformer}.
The token sparsification is applied after RoPE processes the attention keys, so the selected tokens retain their relative positions in the references.
Besides flexibility, one motivation for parallel position is to avoid the ``lost in the middle" phenomenon \cite{liu2023lost}, such that if the references are positioned serially, then the ones in the middle are likely to be ignored.
Similarly, token sparsification also helps to alleviate this issue by making the attention more focused on the important tokens.
We note that designs analogous to the parallel position have been used to improve in-context learning \cite{ratner2022parallel} and long-context modeling \cite{bertsch2024unlimiformer}.

\subsection{Memory Sparsification and Storage}
\label{sec:sparse memory}

One of the greatest challenges for explicit memories is that the attention key-values occupy too much space.
They not only demand more disk space, which could be costly, but also occupy GPU memory during inference, which could harm the batch size and thus the throughput of LLM generation.
An intense compression is needed to save space.
The full attention key tensor (or value tensor) for each reference has shape (layers, key-value heads, tokens, head dimension), so we compress all four dimensions.

Regard layers, we only set the first half of the attention layers to be memory layers, i.e. layers that produce and attend to explicit memories (\ref{eq:memory attention}), while the second half remain as the usual attention layers.
Note that Remark \ref{remark:induction_to_architecture} suggests that it is usually the attention heads in the middle of the LLM that attend to the references.
So it seems that appointing the middle attention layers (e.g. the ones within the $25\%$ to $75\%$ depth range) to be memory layers is a more sensible choice.
This heuristic is supported by the observations in \cite{wu2024retrieval,fang2024unimem} that the attention to the distant context usually takes place in the middle layers.

Regarding heads, we set all key-value heads at each memory layer to be memory heads.
We reduce their amount by grouped query attention (GQA) \cite{ainslie2023gqa}, letting each key-value head be shared by multiple query heads, and obtain 20\% sparsity (8 versus 40 heads).
It is worth mentioning that, besides GQA and memory layers, another approach is to select a small subset of heads that are most helpful for reading memories, and this selection does not have to be uniform across layer.
We describe several methods for selecting memory heads in Remark \ref{remark:head}.

Regarding tokens, we select 8 tokens out of 128 for each key-value head.
We choose a high level of sparsity, since Remark \ref{remark:induction_to_architecture} indicates that the attention from the context to the references are expected to be concentrated on very few tokens.
Note that the selected tokens are in general different among heads, so in principle their union could cover a lot of tokens.
For each head $h$ at layer $l$, the selection uses top-8 over the attention weight
\begin{equation*}
w^{l,h}_j = \sum_{i=0}^{127} \tilde{a}^{l,h}_{i,j}, \quad \tilde{a}^{l,h}_{i,j} = \text{softmax}_j\Big(\frac{X^{l,h}_i W^{l,h}_Q (X^{l,h}_j W^{l,h}_K )^T }{\sqrt{d_h}}\Big)
\end{equation*}
which measures the importance of a token by the attention received from all tokens.
The BOS tokens and paddings do not participate in the the computation of the weights.
These attention weights $\tilde{a}$ are different from the usual ones, such that there is no causal mask or position encoding involved.
The consideration is that since the explicit memories are prepared before any inference, the selection can only depend on the reference itself instead of any context texts.
The removal of causal mask and position encoding ensures that tokens at any position has an equal chance to receive attention from others.
To speed up computation, we adopt the following approximate weights in our implementation, although in retrospect this speedup is not necessary.
\begin{equation*}
w^{l,h}_j = \sum_{i=0}^{127} \exp\Big(\frac{X^{l,h}_i W^{l,h}_Q (X^{l,h}_j W^{l,h}_K )^T }{\sqrt{d_h}}\Big)
\end{equation*}
Similar designs that sparsify tokens based on attention weights have been adopted in long-context modeling to save space \cite{liu2023scissorhands,zhang2024h2o}.

Regarding head dimension, we optionally use a vector quantizer to compress each of the key and value vectors using residual quantizations \cite{chen2010approximate} built with FAISS \cite{douze2024faiss}.
The compression rate is $80/7 \approx 11.4$.
During inference, the retrieved memories are first loaded from drives, and then decompressed by the vector quantizer before being sent to GPU.
The evaluations in Section \ref{sec:eval-general} indicate that this compression has negligible influence on the performance of \name{}.
More details can be found in Appendix \ref{appendix:quantizer}.

Hence, the total sparsity is 160 or 1830 (without or with vector compression).
Originally, the explicit memory bank would have an enormous size of 7.17PB or equivalently 7340TB (given the model shape described in Section \ref{sec:model shape} and saved in bfloat16).
Our compression brings it down to 45.9TB or 4.02TB (without or with vector compression), both acceptable for the drive storage of a GPU cluster.

To deploy the \name{} model on end-side devices such as smart phones and laptops, one can place the explicit memory bank and the vector index on a cloud server, while the devices only need to store the model parameters and the decoder of the vector quantizer.
During inference, to perform retrieval, the model on the end-side device sends the query vector to the cloud server, which then searches the index and returns the compressed memories.
The speed test of this deployment is recorded in Section \ref{sec:speed}.

\begin{remark}
\label{remark:head}
If one wants to finetune a pretrained LLM into a \name{} model, there are several ways to select a small but effective subset of attention heads (among all heads at all layers) for memory heads (\ref{eq:memory attention}).
Methods such as \cite{wu2024retrieval,fang2024unimem} are proposed to identify the heads that contribute the most to long-context modeling by retrieving useful information from distant tokens, and usually these special heads account for only $<10\%$ of the total heads.
Here we also propose a simple method for selecting memory heads:
Given the validation subsets of a representative collection of evaluation tasks, one can measure the average performance $s_h$ for a modified version of the LLM for each attention head $h$.
The modification masks the distant tokens for head $h$ so it can only see the preceding 100 tokens and the BOS token.
Then, it is reasonable to expect that $s_h$ would be markedly low for a small subset of heads $h$, indicating that they are specialized for long-range attention.
\end{remark}

\begin{remark}
\label{remark:adaptive head}
Actually, Remark \ref{remark:induction_to_architecture} suggests that each reference only needs to be attended to by just one attention head, although in general this special head may be different among the references.
Thus, it seems a promising approach to apply adaptive sparsity not only to token selection, but also to the memory heads, namely each reference is routed to one or two heads (analogously to MoE), and its explicit memory is produced and read by these heads.
Such design if feasible can further boost the sparsity of explicit memory and save much more space.
\end{remark}


\subsection{Model Shape}
\label{sec:model shape}

As discussed in Section \ref{sec:memory}, the specific knowledges can be externalized to explicit memories, and thus to minimize the total cost (\ref{eq:cost}), the model parameters (or implicit memory) only need to store abstract knowledges and the subset of specific knowledges that are frequently used.
The shape of our model, i.e. (the number of Transformer blocks $L$, heads $H$, head dimension $d_h$, width of the MLP layers $W$), is chosen to accommodate this desired knowledge distribution.
Informally speaking, given a fixed parameter size $P$, the shape maximizes the following objective
\begin{equation}
\label{eq:shape objective}
\max_{L, H, d_h, W} \Big\{ \frac{\text{capacity for abstract knowledge}}{\text{capacity for specific knowledge}} ~\Big|~ \text{size}(L, H, d_h, W) \approx P \Big\}
\end{equation}
Here we set $P$ to be 2.4 billion.

Some recent works suggest that the capacities for learning specific knowledges and abstract knowledges are subject to different constraints.
On one hand, \cite{dai2021knowledge} observes that the amount of bits of trivia information (such as a person's name, date of birth and job title) that a LLM can store depends only on its parameter size.
Regardless of $L$ and $H$, the max capacity is always around 2 bits per parameter.

On the other hand, \cite{weiss2021thinking} trains Transformers to learn simple algorithms such as reversing a list and counting the occurrence of each letter.
It is observed that for several such tasks, there exists a minimum $L_0$ and $H_0$ such that a Transformer with $L\geq L_0$ and $H \geq H_0$ can learn the task with perfect accuracy, whereas the accuracy drops significantly for Transformers with either $L=L_0-1$ or $H=H_0-1$ (given that either $L_0$ or $H_0\geq 2$).
This sharp transition supports the view that the layers and heads of Transformer LLMs can be compared to algorithmic steps, and tasks with a certain level of complexity require at least a certain amount of steps.
It is worth mentioning that the emergent phenomenon \cite{wei2022emergent,srivastava2023imitation} of LLMs can also be explained by this view and thus adds support to it, although it may not be the only explanation.

By Definition \ref{def:specific}, the abstract knowledges are expected to be circuits with greater complexity than specific knowledges, since their associated inputs and outputs exhibit greater variability and thus express more complex patterns.
It follows that, in the context of the aforementioned works, the separation of specific and abstract knowledges should be positively correlated with the distinction between trivia information and algorithmic procedures.
Hence, it is reasonable to adopt the approximation that the capacity of an LLM for specific knowledges only depends on its parameter size, whereas the capacity for abstract knowledges depends only on $L$ and $H$.

The informal problem (\ref{eq:shape objective}) reduces to the maximization of $L$ and $H$ given a fixed parameter size.
However, we are left with two ambiguities:
first, this formulation does not specify the ratio between $L$ and $H$,
and second the head dimension $d_h$ and MLP width $W$ cannot be too small as the training may become unstable.
Regarding the second point, our experiments indicate that pretraining becomes more unstable with increased spikes if $d_h \leq 64$, so we set $d_h = 80$ (though it needs to be pointed out that the loss spikes may not be solely attributed to the choice of $d_h$, and high-quality data for instance may stabilize training and allow us to choose a smaller $d_h$).
Also, the MLP width $W$ is set to be equal to the hidden dimension $d=H d_h$.
Regarding the first point, controlled experiments (Figure \ref{fig:compare 2B shape}) indicate that the loss decreases slightly more rapidly with $L:H\approx 1$ than with other ratios, so we adopt this ratio.

\begin{figure}[h]
\centering
\includegraphics[scale=0.32]{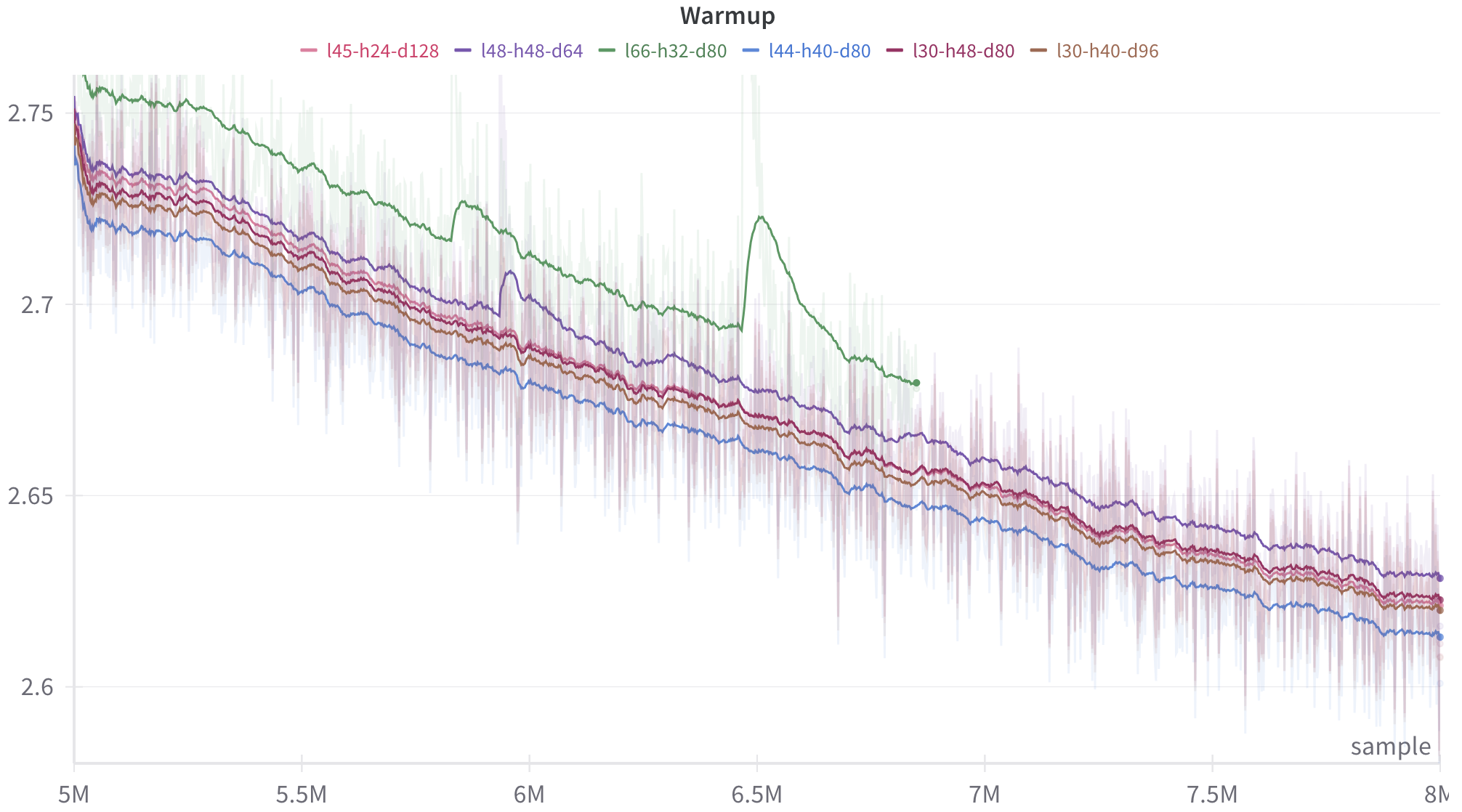}
\caption{Comparison of the training losses of models with different shapes, whose parameter sizes range in $2.1\sim2.4$B.
The legend l44h40d80 denotes $L=44,H=40,d_h=80$, and the $x$-axis denotes the amount of training samples.
Nevertheless, this comparison is not definite, since this is only the warmup stage of our training scheme (Section \ref{sec:two-stage pretrain}) and the ranking may change in the continual train stage when explicit memory is introduced.}
\label{fig:compare 2B shape}
\end{figure}

In addition, as discussed in Section \ref{sec:sparse memory}, our model uses grouped query attention (GQA), so the number of key-value heads $H_{kv}$ is set to be $8$, which is the usual choice for GQA.
The MLP layers are gated two-layer networks without bias, which are the default choice in recent years \cite{touvron2023llama,bai2023qwen,chowdhery2022palm,almazrouei2023falcon}.

Finally, the model shape is set to be $L=44, H=40, H_{kv}=8, d_h=80, W=3200$, with the total non-embedding parameter size being $2.4$B.

\subsection{Training Designs}
\label{sec:train design}

Similar to our architecture design, the design of our training scheme focuses on learning abstract knowledges.
The goal is to reduce the training compute, as the LLM no longer needs to memorize many of the specific knowledges.
This shift in learning objective implies that all the default settings for pretraining LLMs may need to be redesigned, as they were optimized for the classical scenario when the LLMs learn both abstract and specific knowledges.

\begin{enumerate}
\item Data:
Ideally, the pretraining data should have a high concentration of abstract knowledges and minimum amount of specific knowledges.
It is known that LLM pretraining is very sensitive to the presence of specific knowledges.
For instance, \cite{huang2024unified} observes that a small model can master arithmetic (e.g. addition of large numbers) if trained on clean data.
However, if the training data is mixed with trivial information (e.g. random numbers), then the test accuracy stays at zero unless the model size is increased by a factor of 1500.
It suggests that training on specific knowledges significantly inhibits the learning of abstract knowledges, and may explain why emergent abilities \cite{wei2022emergent} are absent from small models.
Notably, the Phi-3 model \cite{abdin2024phi3} is pretrained with a data composition that closely matches our desired composition.
Although the technical details are not revealed, it is stated that they filter data based on two criteria:
the data should encourage reasoning, and should not contain information that is too specific.

\item Initialization: \cite{zhang2024initialization} observes that initializing Transformer parameters with a smaller standard deviation ($d^c$ with $c<-1/2$ instead of the usual $\Theta(d^{-1/2})$ \cite{glorot2010understanding,he2015delving}) can encourage the model to learn compositional inference instead of memorization.
Specially, an arithmetic dataset is designed with a train set and an out-of-distribution test set, which admits two possible answers.
One answer relies on memorizing more rules during training, while the other requires an understanding of the compositional structure underlying these rules.
The proposed mechanism is that training with smaller initialization belongs to the condensed regime that encourages sparse solutions, contrary to training with large initialization that belongs to the kernel regime or critical regime \cite{luo2021phase,chen2023phase}.

\item Weight decay: \cite{power2022grokking,pearce2023machine} observe that using a larger weight decay coefficient (i.e. greater than the usual range of $0.001\sim0.1$) can guide LLMs to favor generalization over memorization, and accelerate the learning of generalizable solutions.
They consider settings that exhibit grokking \cite{power2022grokking} such that training would transit from perfect train accuracy and zero test accuracy to perfect test accuracy, and generalization ability is measured by how quickly this transition occurs.
Moreover, theoretically speaking, it is expected that training generative models needs stronger regularization than training regression models, in order to prevent the generated distributions from collapsing onto the training data and become trivial \cite{yang2022mathematical}.

\end{enumerate}

In summary, it is recommendable to pretrain the \name{} model with a data composition that emphasizes abstract knowledges and minimizes specific information, a smaller initialization for parameters, and a larger weight decay coefficient.

Since this work is only a preliminary version of \name{}, we decide to stick with the conventional setting for training and have not experimented with any of these ideas.
We look forward to incorporating these designs in future versions of the \name{} model.

\subsection{Two-stage Pretrain}
\label{sec:two-stage pretrain}

The \name{} model learns to write and read explicit memories during pretraining.
The training data is prepended with retrieved references;
the model encodes these references into explicit memories in real time, and integrates them into the self-attention computation of the training data.

Unexpectedly, our pretraining consists of two stages, which we name as warmup and continual train.
Only the continual train stage involves explicit memories, while the warmup stage uses the same format as ordinary pretraining.
Our motivation is depicted in Figure \ref{fig:two-stage reason}.
We observe that pretraining with explicit memories from the beginning would render the memories useless, as there appears to be no gain in training loss compared to ordinary pretraining.
Meanwhile, given a checkpoint from ordinary pretraining, continual training with explicit memory exhibits a visible decrease in training loss.
This comparison implies that a memory-less warmup stage might be necessary for pretraining a \name{} model.
One possible explanation for this phenomenon is that in the beginning of pretraining, the model is too weak to understand and leverage the explicit memories it generates.
Then, to reduce distraction, the self-attention layers might learn to always ignore these memories, thus hindering indefinitely the development of explicit memory.

\begin{figure}[h]
\centering
\subfloat{\includegraphics[width=0.48\textwidth]{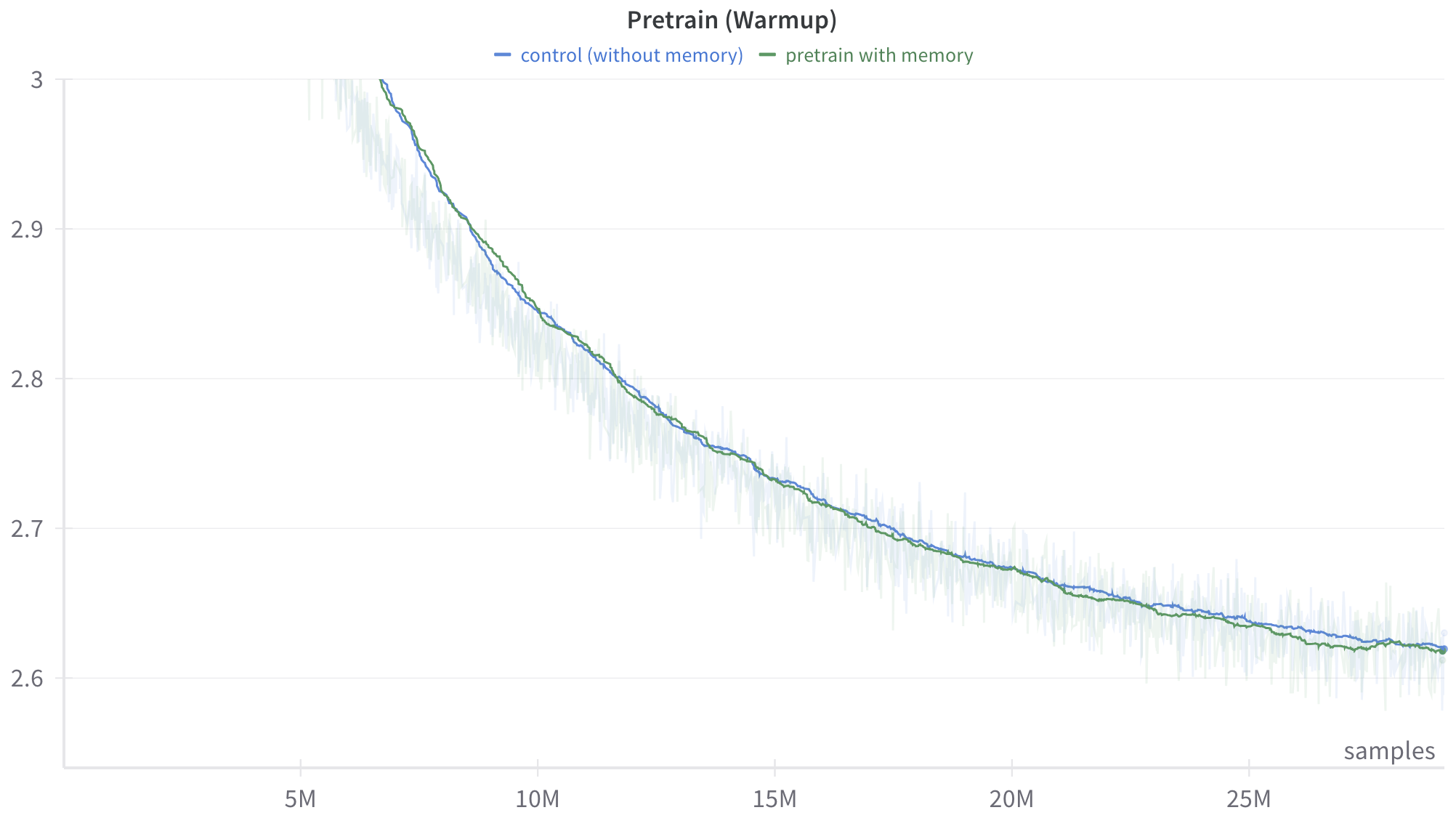}}
\quad
\subfloat{\includegraphics[width=0.48\textwidth]{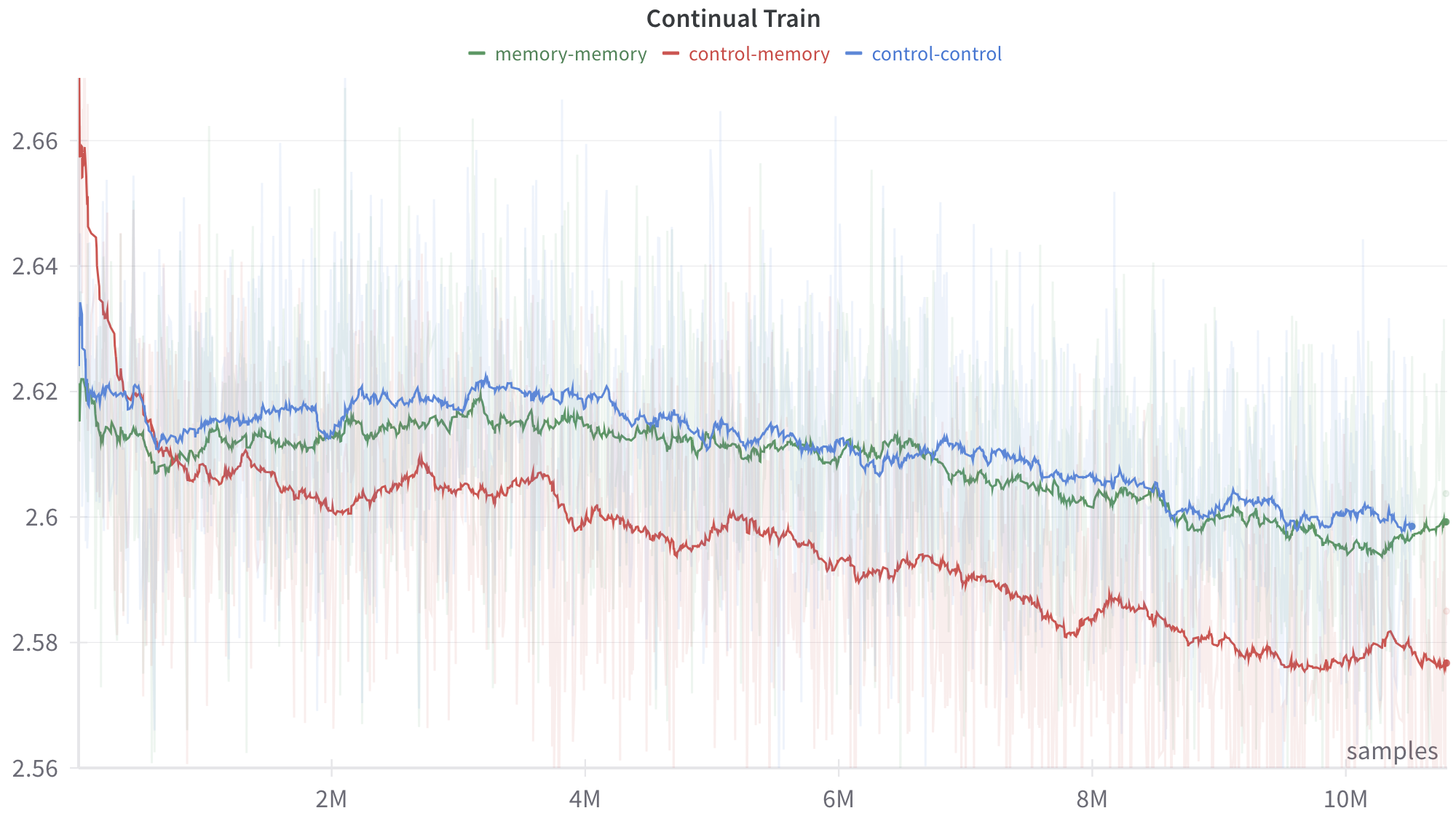}}
\caption{Left: Comparison of the warmup stage (training from scratch) with and without explicit memory.
The blue and green curves are trained without and with explicit memories, respectively.
Right: Comparison of the continual train stage.
The blue and green curves are continual trained from their warmup checkpoints, and the red curve is initialized with the warmup checkpoint of the blue curve and continual trained with explicit memory.
These plots indicate that pretraining a \name{} model requires a memory-less warmup stage.
These experiments use a smaller model with 0.92B non-embedding parameters $(L=40,H=32,d_h=64)$.
The warmup stage uses 60B data and the continual train stage uses 22B.}
\label{fig:two-stage reason}
\end{figure}

Another modification is to reduce the cost of continual train.
Recall from Section \ref{sec:inference} that during inference, each 64-token chunk attends to five explicit memories, or equivalently five 128-token references if using cold start, increasing the amount of input tokens by 10 times.
The inference process avoids the cost of memory encoding by precomputation or warm start, but for the continual train, the references need to be encoded in real time.
Our solution is to let the chunks share their references during training to reduce the total number of references in a batch.
Specifically, each chunk of a training sequence retrieves only one reference, and in compensation, attends to the references of the previous four chunks, besides its own reference.
Each train sequence has length 2048 and thus 32 chunks, so it is equipped with $32\times128=4096$ reference tokens.
The hidden features of these reference tokens are discarded once passing the last memory layer, since after that they no longer participate in the update of the hidden feature of the train tokens.
Hence, each continual train step takes slightly more than twice the amount of time of a warmup step.

It is necessary to avoid information leakage when equipping the training data with references
(i.e. the train sequence and its retrieved references could be the same text), for otherwise training becomes too easy and the model would not learn much.
Previously, Retro \cite{borgeaud2022retro} requires that no train sequence can retrieve a reference from the same document, but this criterion may be insufficient since near-identical paragraphs may appear in multiple documents.
Thus, we require that no train sequence can be accompanied by a reference sequence that has $>90\%$ overlap with it.
The overlap is measured by the length of their longest common subsequence divided by the length of the reference length.
Specially, given any train sequence $\mathbf{t}$ and reference $\mathbf{r}$, define their overlap by
\begin{align}
\label{eq.overlap}
\begin{split}
\text{overlap}(\mathbf{t}, \mathbf{r}) := \frac{1}{|\mathbf{r}|} \max \big\{ N &~\big|~ \exists 1\leq i_1 < \dots < i_N \leq |\mathbf{t}| ~\text{and}~ \exists 1\leq j_1 < \dots < j_N \leq |\mathbf{r}|\\
&~\text{and}~ |i_N-i_1| \leq 2|\mathbf{r}|, ~\text{such that}~ \mathbf{t}_{i_k} = \mathbf{r}_{j_k} ~\text{for}~k=1,\dots N \big\}
\end{split}
\end{align}
The constraint $|i_N-i_1| \leq 2|\mathbf{r}|$ ensures that the overlap is not over-estimated as $|\mathbf{t}|\to\infty$.

\section{Pretraining Data}
\label{sec:data}

This section describes the procedures for collecting and filtering our pretraining dataset and knowledge base (or reference dataset).

\subsection{Data Collection}

The pretrain data is gathered from English and Chinese text datasets, mostly publicly available collections of webpages and books.
We also include code, SFT data (supervised finetuning), and synthetic data.

Specially, the English data mainly consists of RedPajamaV2 \cite{together2023redpajama}, SlimPajama \cite{cerebras2023slimpajama} and the Piles \cite{DBLP:journals/corr/abs-2101-00027}, in total 200TB prior to filtering.
The Chinese data mainly comes from Wanjuan \cite{DBLP:journals/corr/abs-2308-10755}, Wenshu~\cite{courtx9996x9875}, and MNBVC \cite{mnbvc}, in total 500TB prior to filtering.
The code data mainly comes from Github, and we take the subset with the highest repository stars.
The SFT data is included since these samples generally have higher quality than the webpages.
We use the same data as in SFT training (Section \ref{sec:SFT}), except that these samples are treated as ordinary texts during pretraining, i.e. all tokens participate in the loss computation, not just the answer tokens.

\subsection{Filtering}

The raw data is filtered with three steps: deduplication, rule-based filtering, and model-based filtering.

First, deduplication is performed with MinHash for most of the datasets.
One exception is RedPajamaV2, which already comes with deduplication labels.

Second, we devise heuristic, rule-based filters analogous to the ones from
\cite{longpre2023pretrainer,rae2021scaling,conneau2019unsupervised}.
The purpose is to eliminate texts that are ostensibly unsuitable for training, such as ones that only contain webpage source codes, random numbers, or incomprehensible shards.
Our filters remove documents with less than 50 words, documents whose mean word lengths exceed 10 characters,
documents with 70\% of context being non-alphabetic characters,
documents whose fractions of unique words are disproportionately high,
documents whose entropy of unigrams is excessively low, and so on.



Finally, we select the subset of data with highest ``quality", a score produced by a finetuned BERT model.
Specially, we sample ten thousand documents and grade them by the XinYu-70B model \cite{li2024newsbench,liang2024uhgeval} with prompt-guided generation.
The prompt asks the model to determine whether the input text is informative and produce a score between $0$ and $5$.
Then, these scores are used to finetune the Tiny-BERT model \cite{jiao2020tinybert}, which has only 14M parameters.
The hyperparameters of this finetuning are optimized with respect to a held-out validation set.
After that, we use this lightweight BERT to grade the entire dataset.




\begin{remark}
\label{remark:data filter}
Recall from Section \ref{sec:train design} that the pretraining data of \name{} should emphasize abstract knowledges and minimize specific knowledges.
The purpose is to not only obtain a lightweight LLM with an ideal distribution of knowledges in accordance with the memory hierarchy (Figure \ref{fig:knowledge_distribution}), but also prevent the specific knowledges from hindering the learning process of the model.
The focus of our prompt on ``informativeness" might be contradictory to this goal, since the selected texts that are rich in information content may contain too many specific knowledges.
For future versions of \name{}, we will switch to a model-based filter favoring texts that exhibit more reasoning and less specifics.
\end{remark}



\begin{figure}
    \centering
    \includegraphics[width=0.6\textwidth]{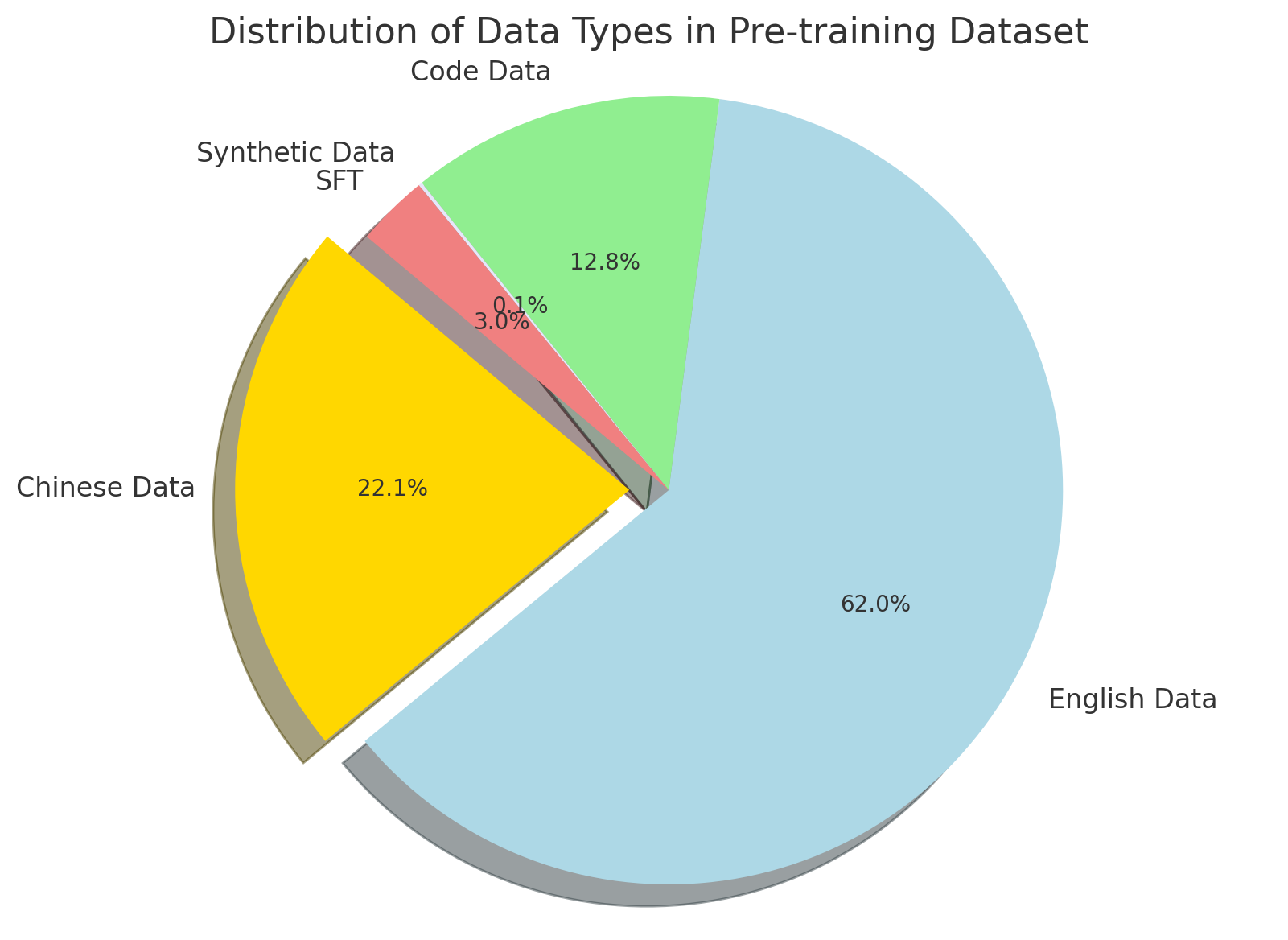}
\caption{Composition of our pretraining dataset.
}
\label{fig:pretrain_data_composition}
\end{figure}

The filtered dataset consists of around four trillion tokens, and its composition is illustrated in Figure \ref{fig:pretrain_data_composition}.

\subsection{Tokenizer} 

Similar to our dataset, our tokenizer mainly consists of Chinese and English tokens.
The English vocabulary comes from the 32000 tokens of the LLaMA2 tokenizer.
We include roughly the same amount of Chinese tokens produced from byte-pair encoding (BPE).
The BPE is trained on a 20GB Chinese corpus that consists of Chinese news and e-books.
After deduplication, the final vocabulary has 60299 tokens.


\subsection{Knowledge Base}
\label{sec:ref}

The knowledge base (or reference dataset) is used during training and inference as the source of explicit memories, as depicted in Figure \ref{fig:opening}.
It consists of reference texts that are split into token sequences with length $\leq 128$, as described in Section \ref{sec:inference}.

Heuristically, a larger knowledge base is always better, as long as it does not contain misinformation,
so it is not surprising that the reference dataset of Retro contains its entire pretrain dataset \cite{borgeaud2022retro}.
Nevertheless, the storage of explicit memories is more costly than plain texts despite our sparsification (Section \ref{sec:sparse memory}), and thus to save storage space, we select a small subset of our pretrain dataset as the knowledge base.

With a focus on high quality data, we include for references the English Wikipedia, WikiHow, the Chinese baike dataset, the subset of English and Chinese books whose titles appear academic, Chinese news, synthetic data and high quality codes.
These texts are tokenized and split into chunks of 128 tokens, resulting in $1.1\times 10^8$ references in total.

One may be curious whether our knowledge base may contain some of the evaluation questions, rendering our evaluation results (Section \ref{sec:eval-general}) less credible.
To prevent such leakage, we include in our evaluation code a filtering step, such that for each evaluation question, if a retrieved reference has an overlap with the question that exceeds a threshold, then it is discarded.
This deduplication is analogous to the one used when preparing for continual train (Section \ref{sec:two-stage pretrain}),
with the overlap measured by (\ref{eq.overlap}).
The threshold $2/3$ is chosen since we observe that typically a reference that contains a question would have an overlap $\geq 80\%$, while a relevant but distinct reference would have an overlap $\leq 40\%$.

\begin{remark}
\label{remark:ref select}
Currently, the compilation of the knowledge base is based on human preference.
For future versions of \name{}, we plan to take a model-oriented approach and measure the fitness of a candidate reference by its actual utility,
e.g. the expected decrease in the validation loss of the LLM conditioned on this reference being retrieved by a random validation sample.
\end{remark}

\section{Pretrain}  
\label{sec:pretrain}

This section describes the details of the pretraining process.
The two-stage pretrain and memory-augmented data follow the designs introduced in Section \ref{sec:two-stage pretrain}.
As an interpretation, the \name{} model during the warmup stage develops its reading comprehension, which is necessary during the continual train stage for initiating memory formation.

\subsection{Set-up}

Training is conducted with the Megatron-DeepSpeed package \cite{megatron-deepspeed} and uses mixed-precision training with bfloat16 model parameters, bfloat16 activations, and float32 AdamW states.
The batch size is around 4 million training tokens with sequence length 2048, not including the reference tokens.
The weight decay is the common choice of 0.1.

We adopt the ``warmup-stable-decay" learning rate schedule of MiniCPM \cite{hu2024minicpm}, which is reportedly better than the usual cosine schedule in term of training loss reduction.
The learning rate linearly increases to the maximum value, then stays there for the majority of training steps, and finally in the last 10\% steps decays rapidly to near zero.
Our short-term experiments confirm the better performance of this schedule.
Nevertheless, frequent loss spikes and loss divergences are encountered during the official pretraining, so we have to deviate from this schedule and manually decrease the learning rate to stabilize training.

Originally, it is planned that both the warmup and continual train stages go through the entire 4T token pretrain dataset (Section \ref{sec:data}).
Due to the irremediable loss divergences, both stages have to be terminated earlier.

\subsection{Warmup Stage}
\label{sec:warmup}

\begin{figure}[h]
\centering
\subfloat{\includegraphics[width=0.475\textwidth]{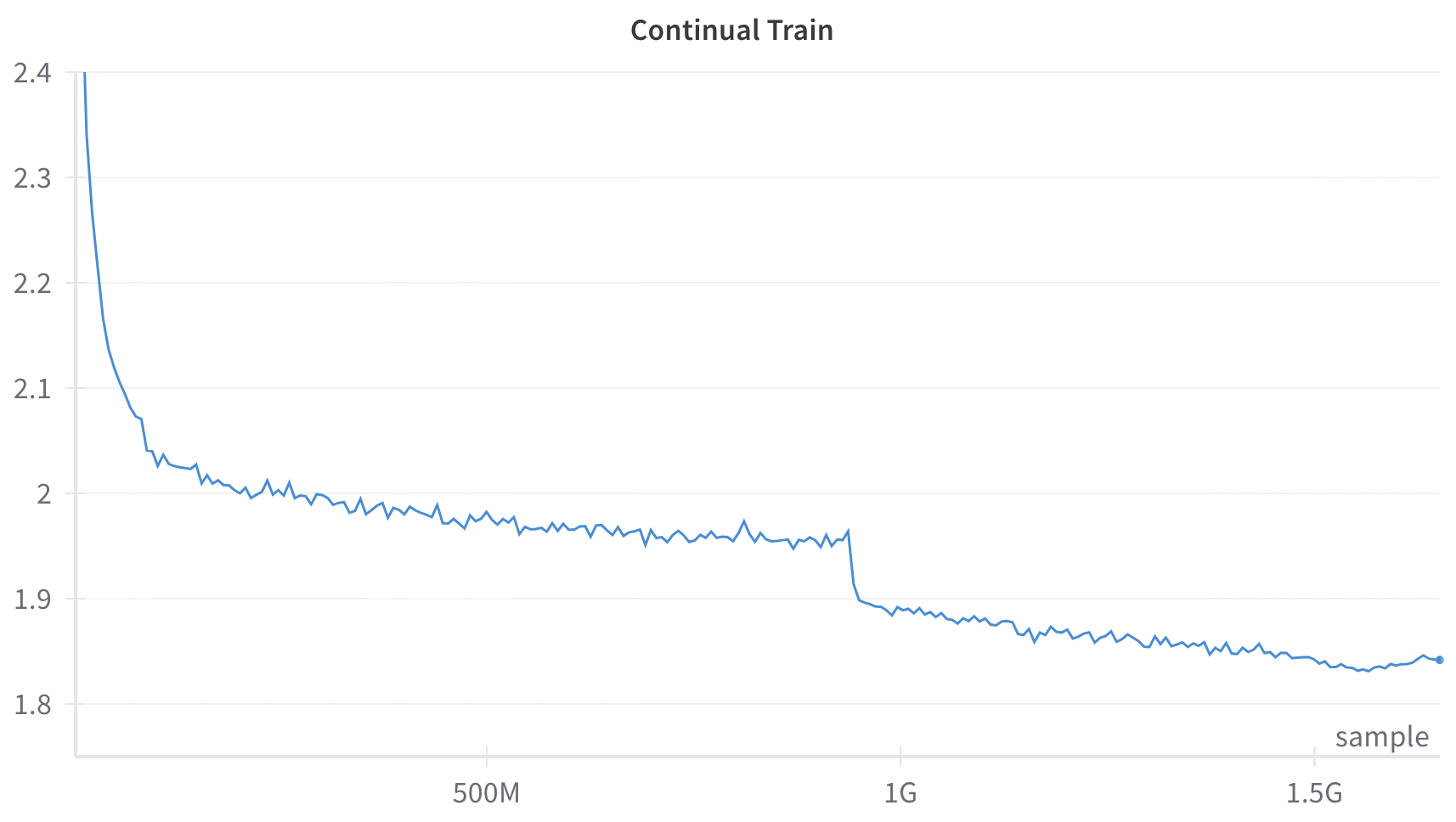}}
\quad
\subfloat{\includegraphics[width=0.475\textwidth]{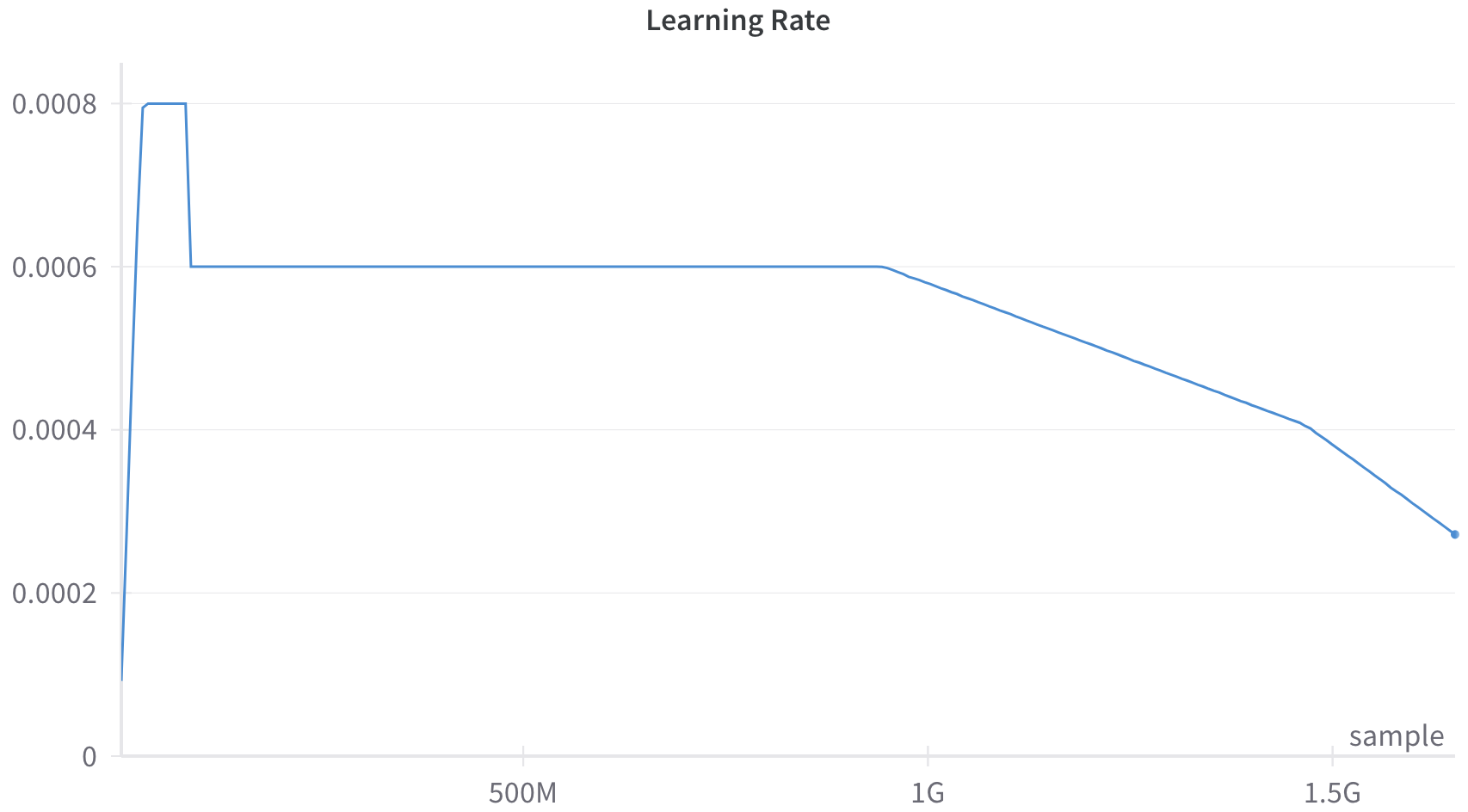}}
\caption{The warmup stage without explicit memory.
Left: Training loss.
Right: Learning rate schedule.}
\label{fig:warmup}
\end{figure}

The training loss and learning rate schedule are plotted in Figure \ref{fig:warmup}.
Whenever severe loss divergence occurs, we restart from the last checkpoint before the divergence with a smaller learning rate, and thus the divergences are not shown in the figure.
Eventually, the training terminates at around 3.1T tokens, when reducing the learning rate can no longer avoid loss divergence.


\subsection{Continual Train Stage}
\label{sec:continual train}

\begin{figure}[h]
\centering
\subfloat{\includegraphics[width=0.45\textwidth]{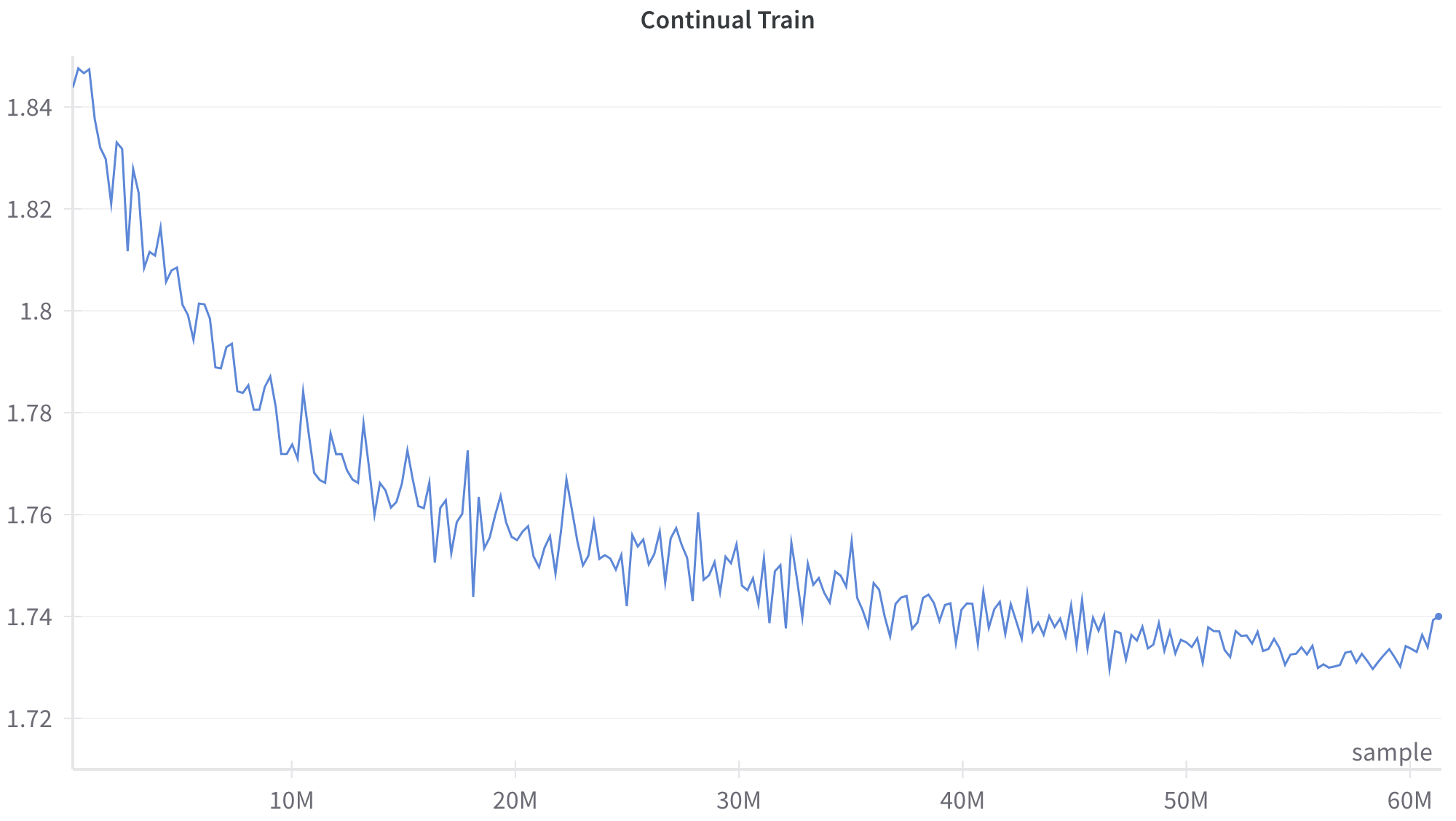}}
\quad
\subfloat{\includegraphics[width=0.45\textwidth]{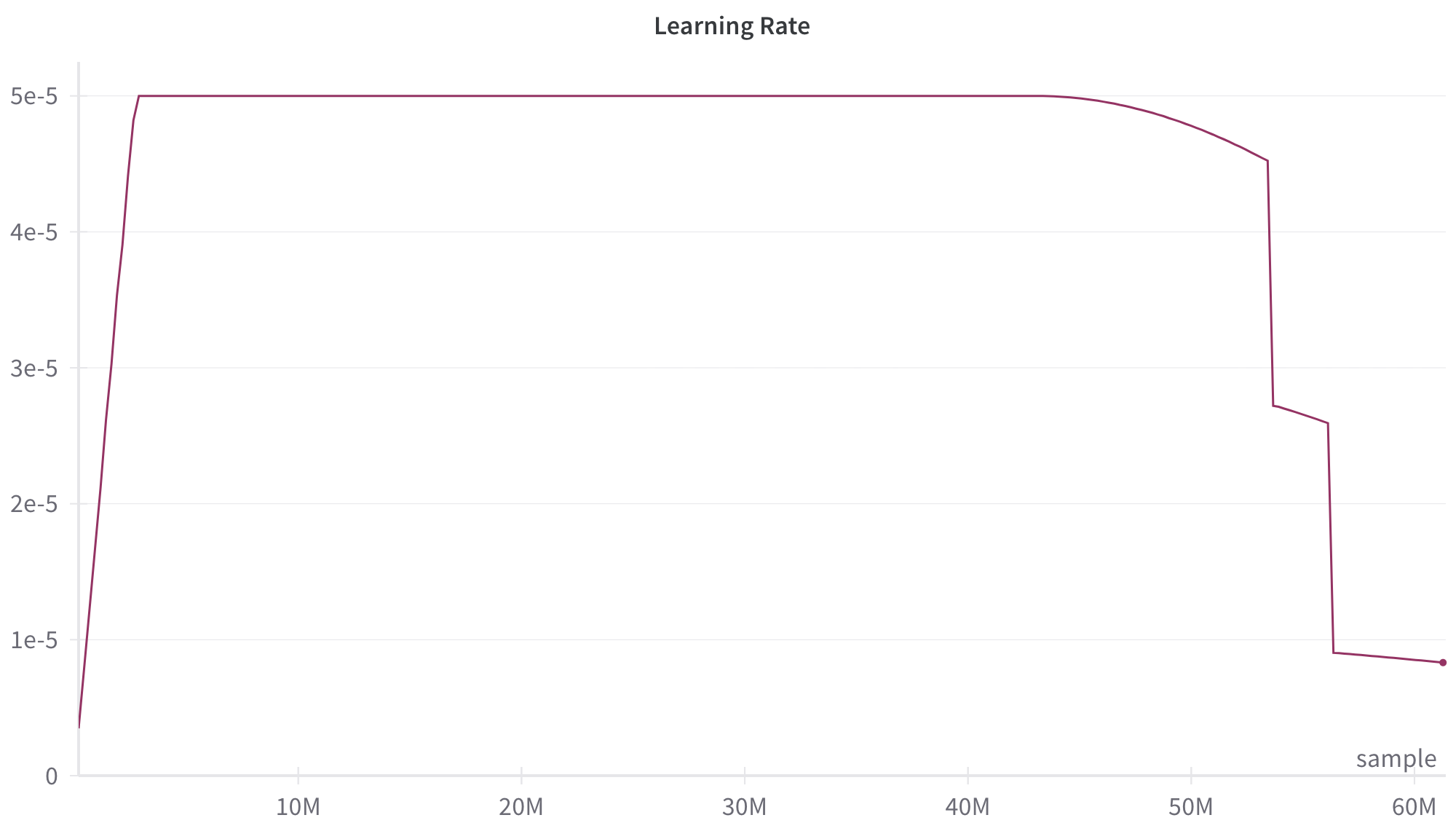}}
\caption{The continual train stage with explicit memory.
Left: Training loss.
Right: Learning rate schedule.}
\label{fig:continual train}
\end{figure}

The explicit memories enter into the \name{} model at this stage.
The training steps are slower since the model needs to encode the references retrieved for the pretrain data to explicit memories in real time, and each step takes a bit more than twice the time of a warmup step.
The training loss and learning rate schedule are plotted in Figure \ref{fig:continual train}.

The loss divergence soon becomes irremediable at around 120B training tokens, much shorter than the planned 4T tokens, and training has to stop there.
One possible cause is that the continual train is initialized from the latest warmup checkpoint, which is located immediately before the break down of the warmup stage, and thus is already at the brink of divergence.
The smaller learning rate of continual train delays the onset of divergence but not for long.

\section{Fine-tuning and Alignment}  
\label{sec:finetune}

This section describes our model finetuning, specifically supervised finetuning (SFT) and direct preference optimization (DPO).

\subsection{Supervised Finetuning}  
\label{sec:SFT}

Analogous to the StableLM model \cite{bellagente2024stable},
our \name{} model is finetuned on a diverse collection of SFT datasets.
We use the following datasets, which are publicly accessible on the Hugging Face Hub:
UltraChat \cite{ding2023enhancing}, WizardLM \cite{xu2023wizardlm}, SlimOrca \cite{SlimOrca}, ShareGPT \cite{wang2023openchat}, Capybara \cite{daniele2023amplify-instruct}, Deita \cite{liu2023what}, and MetaMathQA \cite{yu2023metamath}.
We also include synthetic data with emphasis on multi-round chat, mathematics, commonsense and knowledge.
Each training sample consists of one or more rounds of question and answer pairs.
We remove any sample with more than eight rounds.
The final composition is listed in Table \ref{tab:finetune-datasets}.


\begin{table}[ht]
\footnotesize
\centering
\begin{tabular}{l c r}
\toprule
\textbf{Dataset} & \textbf{Source} & \textbf{Number of Samples} \\[1ex]
\midrule
UltraChat & HuggingFaceH4/ultrachat\_200k & 194409 \\ [0.2ex]
WizardLM & WizardLM/WizardLM\_evol\_instruct\_V2\_196k & 80662 \\ [0.2ex]
SlimOrca & Open-Orca/SlimOrca-Dedup & 143789 \\ [0.2ex]
ShareGPT & openchat/openchat\_sharegpt4\_dataset & 3509 \\ [0.2ex]
Capybara & LDJnr/Capybara & 7291 \\ [0.2ex]
Deita & hkust-nlp/deita-10k-v0 & 2860 \\ [0.2ex]
MetaMathQA & meta-math/MetaMathQA & 394418 \\ [0.2ex]
\midrule
Multi-round Chat & synthetic & 20000 \\ [0.2ex] 
Mathematics & synthetic & 20000  \\  [0.2ex] 
Commonsense & synthetic & 150000\\ [0.2ex] 
Knowledge & synthetic & 270000 \\ [0.2ex] 
\midrule
\end{tabular}
\caption{Composition of SFT dataset.}
\label{tab:finetune-datasets}
\end{table}



The training process uses the cosine learning rate schedule with a max
learning rate of $5 \times 10^{-5}$ and a $10\%$ linear warmup phase.
The weight decay is 0.1, batch size is 512, and max sequence length is 2048 tokens.
Finetuning is performed for 3 epochs.

\subsection{Direct Preference Optimization} 
\label{sec:DPO}

The \name{} model is further finetuned by DPO \cite{rafailov2024direct}, to align with human preference and improve its conversation skills.
The DPO dataset consists of general conversations (UltraFeedback Binarized \cite{tunstall2023zephyr}), math questions (Distilabel Math \cite{argilla2023distilabel}) and codes questions (Synth Code \cite{pvduy2023synth}).
The training uses the cosine learning rate schedule with max lr $4\times 10^{-6}$.
The inverse temperature $\beta$ of the DPO loss is set to $0.01$.
The improvement from DPO is displayed in Section \ref{sec:conversation}.


\section{Evaluation} 
\label{sec:eval}

We evaluate the general abilities (benchmark tasks), conversation skills, professional abilities (law and medicine), and facutality \& hallucination of the \name{} model.
We also measure its decoding speed.
Our model is compared with SOTA LLMs of similar and larger sizes, as well as RAG models.

\subsection{General Abilities} 
\label{sec:eval-general}

To evaluate the general abilities of \name{}, we adopt all tasks from the Huggingface leaderboard and also include two Chinese tasks.
Most of the results are displayed in Table~\ref{tab:basic-results}, while TruthfulQA is listed in Table~\ref{tab:hallu-results}.
All results are obtained in bfloat16 format, using the lm-evaluation-harness package \cite{eval-harness} and the configuration of HuggingFace Open LLM leaderboard \cite{open-llm-leaderboard}, i.e. the number of few-shot examples and grading methods.

As described in Section \ref{sec:ref}, to prevent cheating, a filtering step is included in the retrieval process so that the model cannot copy from references that resemble the evaluation questions.

\begin{table}[h]
\centering
\footnotesize
\rowcolors{2}{Tue-red!10}{white}
\begin{tabular}{cccccc>{\centering\arraybackslash}p{0.075\linewidth}ccc}
\toprule
                &         &         & \multicolumn{5}{c}{English}& \multicolumn{2}{c}{Chinese}\\ 
                                      \cmidrule(lr){4-8}                                                  \cmidrule(lr){9-10}
LLM& Size & Avg.    & ARC-C& HellaSwag& MMLU& Winogrande& GSM8k& CEVAL &CMMLU\\
\midrule
Falcon-40B& 41B& 55.75& 61.86& 85.28& 56.89& 81.29& 21.46& 41.38& 42.07\\
Llama2-7B-Chat& 6.5B& 46.87& 52.90& 78.55& 48.32& 71.74& 7.35& 34.84& 34.40\\
Llama2-13B-Chat& 13B& 51.78& 59.04& 81.94& 54.64& 74.51& 15.24& 38.63& 38.43\\
Llama3-8B-it& 7.0B& 65.77& 62.03& 78.89& 65.69& 75.77& 75.82& 50.52& 51.70\\
Vicuna-13B-v1.5& 13B& 52.02& 57.08& 81.24& 56.67& 74.66& 11.30& 41.68& 41.53\\
Mistral-7B-v0.1& 7.0B& 59.15& 59.98& 83.31& 64.16& 78.37& 37.83& 45.91& 44.49\\
Gemma-2B-it& 2.0B& 36.64& 38.02& 40.36& 55.74& 35.29& 55.88& 8.26& 29.94\\
Gemma-7B-it& 7.8B& 47.23& 51.45& 71.96& 53.52& 67.96& 32.22& 27.93& 25.70\\
MiniCPM-2B-SFT& 2.4B& 54.37& 47.53& 71.95& 51.32& 67.72& 45.26& 48.07& 48.76\\
Phi-2& 2.5B& 55.70& 61.09& 75.11& 58.11& 74.35& 54.81& 34.40& 32.04\\
ChatGLM3-6B& 5.7B& 54.62& 41.38& 66.98& 50.54& 64.25& 51.25& 54.01& 53.91\\
Baichuan2-7B-Chat& 6.5B& 55.16& 52.73& 74.06& 52.77& 69.77& 28.28& 53.12& 55.38\\
Qwen1.5-1.8B-Chat& 1.2B& 49.67& 38.74& 60.02& 45.87& 59.67& 33.59& 55.57& 54.22\\
Qwen1.5-4B-Chat& 3.2B& 58.15& 43.26& 69.73& 55.55& 64.96& 52.24& 61.89& 59.39\\
Qwen1.5-7B-Chat& 6.5B& 64.80&	56.48& 79.02& 60.52& 66.38& 54.36& 68.20& 68.67\\
\hline
\name{}-SFT & 2.4B& 63.31& 58.11& 80.51& 59.68& 74.51& 52.84& 59.29& 58.24\\
with vector compression & 2.4B& 63.33& 57.94& 80.65& 59.66& 75.14& 52.24& 59.66& 58.05\\
without memory & 2.4B& 60.80& 57.42& 73.14& 57.29& 74.35& 51.33& 56.32& 55.72\\
\bottomrule
\end{tabular}
\caption{Few-shot evaluation of general abilities.
The model sizes only include non-embedding parameters.}
\label{tab:basic-results}
\end{table}

The results of our model without using explicit memory is included, which indicates that explicit memory boosts the average score by $2.51\%$.
In comparison, the score difference between Llama2-7B and 13B is $4.91\%$ while the latter has twice the amount of non-embedding parameters.
Thus, it reasonable to say that explicit memory can increase the ``effective model size" by $2.51 / 4.91 \approx 51.1\%$.
(Also, the score difference between Qwen-1.8B and 4B is $8.48\%$ while the latter has $167\%$ more non-embedding parameters.
With respect to this scale, explicit memory increases the ``effective model size" by $1.2.51 / 8.48 \times 1.67 \approx 49.4\%$.)

We also include the results of \name{} with vector compression (Section \ref{sec:sparse memory}).
Even though the key-value vectors of the explicit memories are compressed to $8.75\%$ of their original sizes, the performance of our model does not show any degradation.

Other supplementary evaluations can be found in Appendix \ref{appendix:evaluation}.

\vspace{0.5em}

Next, we compare with a LLM that is pretrained with text retrieval.
Specially, we consider the largest version of the Retro++ model \cite{wangRetro++2023}, Retro++ XXL with 9.5B parameters.
All tasks from Table 6 of \cite{wangRetro++2023} are taken, except for HANS, which is not available on lm-eval-harness, and all tasks are zero-shot.
Similar to Table \ref{tab:basic-results}, \name{} is tested with a filtering threshold of 2/3.
The results are listed in Table \ref{tab:retro++}, where \name{} outperforms the model with much larger parameter size and reference dataset size.

\begin{table}[h]
\footnotesize
\centering
\begin{tabular}{cc|c|cccccccc}
\toprule
\rowcolor{Tue-red!10}
LLM& Param size & Avg.& HellaSwag& BoolQ& Lambada& RACE\\
Retro++ XXL& 9.1B& 61.0& 70.6& 70.7& 72.7& 43.2\\
\name{}-SFT& 2.4B& \textbf{64.7}& 83.3& 80.4& 57.9& 45.3\\
\hline \hline
\rowcolor{Tue-red!10}
 & Reference size & & PiQA& Winogrand& ANLI-R2& WiC\\
 & 330B& & 77.4& 65.8& 35.5& 52.4\\
 & 14.3B& & 76.6& 75.8& 41.6& 56.9\\
\bottomrule
\end{tabular}
\caption{Zero-shot comparison of LLMs pretrained with retrieval.
The scores of Retro++ are taken from \cite{wangRetro++2023}.
The size of a reference dataset is its number of tokens.
The non-embedding parameter size of Retro++ is inferred from its vocabulary size.}
\label{tab:retro++}
\end{table}

\subsection{Conversation Skill} 
\label{sec:conversation}

Next we evaluate the conversation skills of \name{}.
We use MT-Bench (the Multi-turn Benchmark) \cite{zheng2024judging} that consists of multi-round and open-ended questions.
The results are listed in \cref{tab:MT-Bench}, including the \name{} model finetuned by DPO introduced in Section \ref{sec:DPO}.

\begin{table}[h]
\small
\centering
\begin{tabular}{@{}lcc@{}}
\toprule
LLM                      & Size & MT-Bench Score \\ \midrule
Phi-3                    & 3.6B    & 8.38 \\
Mistral-7B-Instruct-v0.2 & 7.0B      & 7.60            \\
Qwen1.5-7B-Chat          & 6.5B      & 7.60            \\
Zephyr-7B-beta           & 7.0B      & 7.34           \\
MiniCPM-2B-DPO           & 2.4B    & 6.89 \\
Llama-2-70B-Chat         & 68B     & 6.86           \\
Mistral-7B-Instruct-v0.1 & 7.0B      & 6.84           \\
Llama-2-13B-Chat         & 13B     & 6.65           \\
Llama-2-7B-Chat          & 6.5B      & 6.57           \\
MPT-30B-Chat             & 30B     & 6.39           \\
ChatGLM2-6B              & 6.1B      & 4.96           \\
Falcon-40B-Instruct      & 41B     & 4.07           \\
Vicuna-7B                & 6.5B      & 3.26           \\
\midrule
\name{}-SFT                   & 2.4B    & 5.31           \\ 
\name{}-DPO               & 2.4B    & 5.80           \\
\bottomrule
\end{tabular}
\caption{MT-Bench scores.
The model sizes only include non-embedding parameters.}
\label{tab:MT-Bench}
\end{table}

We obtain all these scores using GPT-4-0613 as grader, following the single answer grading mode of MT-Bench.
Our model outperforms Vicuna-7B, Falcon-40B-Instruct, and ChatGLM2-6B with fewer parameters.

\subsection{Hallucination and Factuality}

Despite considerable progress, LLMs still face issues with hallucination, leading to outputs that often stray from factual accuracy \cite{rawte-etal-2023-troubling}.
Conceptually, \name{} should be less vulnerable to hallucination, since its explicit memories directly correspond to reference texts, whereas compressing texts into the model parameters might incur information loss.
To evaluate hallucination, we select two English datasets, TruthfulQA \cite{lin-etal-2022-truthfulqa} and HaluEval, and one Chinese dataset \cite{li2023halueval}, HalluQA \cite{cheng2023evaluating}.
TruthfulQA is implemented with lm-evaluation-harness \cite{eval-harness}, while HaluEval and HalluQA are implemented with UHGEval \cite{UHGEval}.
The results are shown in Table \ref{tab:hallu-results}, with \name{} achieving the highest scores on most tasks.


\begin{table}[h]
\centering
\footnotesize
\rowcolors{2}{Tue-red!10}{white}
\begin{tabular}{llllllll}
\toprule
                &         &         & \multicolumn{4}{c}{English}                                       & Chinese \\ 
                                      \cmidrule(lr){4-7}                                                  \cmidrule(lr){8-8}
LLM             & Size & Avg.    & HaluE-QA & HaluE-Dialogue & TruQA-MC1 & TruQA-MC2 & HalluQA \\
\midrule
Falcon-40B      & 41B     & 35.37 & 46.84     & 40.80           & 27.29        & 41.71        & 20.18 \\
Llama2-13B      & 13B     & 28.01 & 23.34     & 31.05           & 25.95        & 36.89        & 22.81 \\
Vicuna-13B-v1.5 & 13B     & 37.07 & 24.93     & 37.35           & 35.13        & 50.88        & N/A   \\
Baichuan2-13B   & 13B     & 37.64 & 46.02     & 45.45           & 26.81        & 39.79        & 30.12 \\
Gemma-7B        & 7.8B    & 37.03 & 50.91     & 48.19           & 20.69        & 46.65        & 18.71 \\
Mistral-7B-v0.1 & 7.0B    & 34.18 & 40.68     & 37.64           & 28.03        & 42.60        & 21.93 \\
Llama2-7B       & 6.5B    & 36.80 & 52.46     & 51.93           & 25.09        & 38.94        & 15.59 \\
Baichuan2-7B    & 6.5B    & 38.63 & \textbf{62.33} & 47.84      & 23.01        & 37.46        & 22.51 \\
ChatGLM3-6B     & 5.7B    & 40.96 & 43.38     & 50.03           & 33.17        & 49.87        & 28.36 \\
Qwen1.5-4B-Chat & 3.2B    & 33.30 &	24.64     &	37.72           & 29.38        & 44.74        & 30.00 \\
Phi-2           & 2.5B    & 38.31 & 50.71     & 39.55           & 31.09        & 44.32        & 25.89 \\
MiniCPM-SFT     & 2.4B    & 36.47 & 49.24     & 47.80           & 24.11        & 37.51        & 23.71 \\
Gemma-2B        & 2.0B    & 38.04 & 53.41     & 52.22           & 24.60        & 39.78        & 20.18 \\
Qwen1.5-1.8B-Chat & 1.2B  & 37.52 & 47.18     & 52.11           & 26.68        & 40.57        & 21.05 \\
\midrule
\name{}-SFT     & 2.4B    & \textbf{48.60} & 56.61     & \textbf{53.91}           & \textbf{38.80}        & \textbf{57.72}        & \textbf{35.96} \\
\bottomrule
\end{tabular}
\caption{Evaluation of hallucination.
HaluE and TruQA denote HaluEval and TruthfulQA, respectively.
Bolded numbers are the best results.
The model sizes only include non-embedding parameters.
Vicuna-13B-v1.5 gets one N/A since that entry is near zero and seems abnormal.}
\label{tab:hallu-results}
\end{table}

\subsection{Professional Tasks}
\label{sec:profession}

One benefit of using explicit memory is that the LLM can easily adapt to new fields and tasks by updating its knowledge base.
One can simply import task-related references into the knowledge base of \name{}, and optionally, convert them to explicit memories in the case of warm start.
Then, the model can perform inference with this new knowledge, skipping the more costly and possibly lossy process of finetuning, and running faster than RAG.
This cost reduction has been demonstrated in Figure \ref{fig:total_cost_usage_2B} and Appendix \ref{appendix:cost}, and could facilitate the rapid deployment of LLMs across various industries.

Besides cost reduction, we need to demonstrate that \name{} can perform no worse than RAG.
We consider two professional tasks in law and medicine.
The legal task consists of multiple-choice questions from the Chinese National Judicial Examination (JEC-QA) dataset \cite{zhong2020JECQA}.
The field-specific references are legal documents from the Chinese national laws and regulations database \cite{flk-npc-gov-cn}.
These references are merged with our general-purpose knowledge base (Section \ref{sec:ref}) for inference.

The medical task consists of the medicine-related questions of C-Eval, MMLU and CMMLU, specifically from the following subsets:
\begin{itemize}
\item C-Eval: clinical medicine, basic medicine
\item MMLU: clinical knowledge, anatomy, college medicine, college biology, nutrition, virology, medical genetics, professional medicine
\item CMMLU: anatomy, clinical knowledge, college medicine, genetics, nutrition, traditional Chinese medicine, virology
\end{itemize}
Our knowledge base is supplemented with medical texts from the open-source medical books dataset \cite{medical-books}.

\begin{table}[h]
\small
\centering
\begin{tabular}{lcclccl}
\toprule
                & \multicolumn{3}{c}{JEC-QA}& \multicolumn{3}{c}{MED}\\ 
                \cmidrule(lr){2-4}                                          \cmidrule(lr){5-7}
LLM & 3 refs& 5 refs & 7 refs& 3 refs& 5 refs &7 refs\\
\midrule
\name{}-2B-SFT& \multicolumn{3}{c}{39.38}& \multicolumn{3}{c}{56.22}\\
MiniCPM-2B-SFT& 38.83& 37.65& 37.94& 53.73& 53.29&52.84\\
Gemma-2B& 28.16& 28.06& 25.29& 42.04& 42.49&42.96\\
Gemma-2B-it& 30.04& 31.13& 29.34& 41.70& 43.24&42.66\\
Llama-2-7B& 28.06& 24.70& 24.90& 45.14& 44.43&37.96\\
Llama-2-7B-Chat& 26.18& 25.10& 25.20& 48.18& 47.29&39.39\\
Phi-2& 25.00& 25.30& 23.32& 50.05& 45.42&45.59\\
Qwen1.5-1.8B-Chat& 42.98& 43.87& 41.50& 52.16& 52.50& 52.16\\
Qwen1.5-4B-Chat& 51.98& 50.49& 50.99& 61.19& 61.02& 61.06\\
\bottomrule
\end{tabular}
\caption{Comparison with RAG on professional tasks.}
\label{tab:domain-results}
\end{table}

The results are shown in Table \ref{tab:domain-results}, and \name{} achieves better performance than most of the models.
All evaluations use 5-shot prompting.
The RAG models retrieve from the same knowledge bases and FAISS indices, except that they receive text references instead of explicit memories.
They only retrieve once for each question, using only the question text for query, so the 5-shot examples do not distract the retrieval.
Since the optimal number of references is not known for these RAG models, we test them for 3, 5, and 7 references per question, and it seems that $3\sim5$ references are optimal.
The usual formatting for RAG is used, i.e. header 1 + reference 1 + reference 2 + reference 3 + header 2 + few-shot examples + question, all separated by line breaks.

The performance plotted in Figure \ref{fig:eval plots} (right) is the average of the scores of the two tasks in Table \ref{tab:domain-results} with five references.

\subsection{Inference Speed} 
\label{sec:speed}

Finally, we evaluate the decoding speed or throughput of \name{}, measured by generated tokens per second.
The results are compared to those of RAG models, to quantify the speedup of explicit memory over text retrieval.

A direct comparison of speeds is uninformative:
The memory hierarchy (Figure \ref{fig:knowledge_distribution}) implies that the \name{} model is more reliant on retrieval to supply knowledge, and naturally \name{} performs retrieval with higher frequency ($5$ references per 64 tokens, possibly higher in future versions).
Therefore, it is necessary to jointly compare performance and speed.
The speed measured in this section is plotted against the retrieval-augmented test accuracy from Section \ref{sec:profession}, resulting in Figure \ref{fig:eval plots} (right).


We measure decoding speed on a A800 GPU, and run all models with Flash Attention \cite{dao2022flashattention}.
All models receive an input of batch size 32 and length 128 tokens, and generate an output with length 128 tokens.
The throughput is computed by $32\times 128$ divided by the time spent.
We test each model 9 times, remove the first record, and take the average of the rest.
\name{} performs $2\times128/64-1=3$ retrievals (the $-1$ means that the first decoded chunk inherits the explicit memories retrieved by the last input chunk).
Each retrieval uses 32 queries to get $32\times 5$ explicit memories.
We consider the warm start scenario, with the explicit memories precomputed and saved to drives.
We implement the worst case scenario, such that the reference ids are reset to be unique after vector search and the memory cache on RAM is disabled, forcing \name{} to load $32\times 5$ memories from drives.
Meanwhile, each RAG model performs one retrieval with query length 64 tokens, receives 5 references for each sample, and inserts them at the beginning of the sample, similar to the setup for Table \ref{tab:domain-results}.

The results are listed in Table \ref{tab:throughput} (local server).
The throughput of these models without retrieval is also provided.

\begin{table}[h]
\footnotesize
\centering
\begin{tabular}{lccccc}
\toprule
& & \multicolumn{2}{c}{\textbf{Local server}} & \multicolumn{2}{c}{\textbf{End-side device}} \\
\cmidrule(lr){3-4} \cmidrule(lr){5-6}
\textbf{LLM} & \textbf{Size} & with retrieval & w/o retrieval & with retrieval & w/o retrieval \\
\midrule
\name{}-2B                & 2.4B & 733.0  & 1131     & 27.6      & 44.36 \\
MiniCPM-2B                & 2.4B & 501.5  & 974.0    & 21.7      & 51.79 \\
Gemma-2B-it               & 2.0B & 1581   & 2056     & 22.0      & 29.23 \\
Gemma-7B-it               & 7.8B & 395.6  & 1008     & 9.5       & 18.61 \\
Mistral-7B-Instruct-v0.1  & 7.0B & 392.9  & 894.5    & 11.1 & 28.7\\
Llama-2-7B-Chat           & 6.5B & 382.8  & 1005     & 10.0& 23.19\\
Llama-2-13B-Chat          & 13B  & 241.1  & 632.5    & 2.5& 5.44\\
Qwen1.5-1.8B-Chat         & 1.2B & 908.2  & 1770     & -& -\\
Qwen1.5-4B-Chat           & 3.2B & 460.7  & 1002     & 22.3& 53.39\\
Qwen1.5-7B-Chat           & 6.5B & 365.8  & 894.5    & -& -\\
Phi-2                     & 2.5B & 622.2  & 1544     & -& -\\
\bottomrule
\end{tabular}
\caption{Inference throughput, measured by tokens per second.}
\label{tab:throughput}
\end{table}

In addition, we study the throughput of these models when they are hosted on an end-side device and retrieve from a knowledge base on a remote server.
Specifically, we use Jetson AGX Orin, and the server uses the vector engine MyScale \cite{myscale}.
The models are run with plain attention, with batch size 1.
To simulate real-world use cases, the input is a fixed text prompt, with approximately 128 tokens, while the exact length can vary among different tokenizers.
The output length is fixed to be 128 tokens.
The results are listed in Table \ref{tab:throughput} (end-side device), and the \name{} model .

\begin{remark}
\label{remark:memory time cost}
Table \ref{tab:throughput} indicates that our \name{}-2B model is $1 - 733/1131 \approx 35.2\%$ slower than the same model without using memory.
This is peculiar considering that reading explicit memories accounts for only a tiny fraction of the total compute:
\begin{equation*}
\frac{2.884\times10^{-3}~\text{TFlops}}{1.264~\text{TFlops}}\approx 0.228\%
\end{equation*}
(The calculations are based on Appendix \ref{appendix:cost}.)
Controlled experiments indicate that the time consumption is mainly due to two sources:
\begin{itemize}
\item Loading the memory key-values from drives to GPU:
This overhead becomes prominent as \name{} retrieves with higher frequency.

\item Python implementation of chunkwise attention:
When encoding a prompt, since each chunk attends to a different set of explicit memories, we use a for loop over the chunks to compute their attentions.
\end{itemize}
They dominate other sources such as computing query vectors by the embedding model and searching the vector index.
We will try to optimize our code to reduce these overheads to be as close as possible to $0.228\%$ of the total inference time, e.g. implement the chunkwise attention with a CUDA kernel.
\end{remark}

\section{Conclusion}
\label{sec:conclude}

The goal of this work is to reduce the cost of LLM training and inference, or equivalently, to construct a more efficient LLM that matches the performance of larger and slower LLMs.
We analyze LLMs from the new perspective of knowledge manipulation, characterizing the cost of LLMs as the transport cost of ``knowledges" in and out of various memory formats.
Two causes of inefficiency are identified, namely the suboptimal placement of knowledges and the knowledge traversal problem.
We solve both problems with explicit memory, a novel memory format, along with a new training scheme and architecture.
Our preliminary experiment, the \name{}-2B model, exhibits stronger abilities and higher speed than many SOTA models with greater sizes as well as RAG models.

For future work, we plan to explore the following directions:
\begin{enumerate}
\item Efficient training with abstract knowledges: Ideally, the training cost of \name{} model should be proportional to the small amount of non-separable knowledges, approaching the learning efficiency of humans.
One approach is to filter the training data to maximize abstract knowledges and minimize specific knowledges (cf. Section \ref{sec:train design} and Remark \ref{remark:data filter}), and preferably the LLM should assess the quality of its own training data and ignore the unhelpful tokens.

\item Human-like capabilities: As described in the introduction, the explicit memory allows for interesting cognitive functions such as handling infinite contexts (conversion of working memory to explicit memory), memory consolidation (conversion of explicit memory to implicit memory), and conscious reasoning (reflection on the memory recall process).
These designs may further improve the efficiency and reasoning ability of \name{}.

\item Compact representation of explicit memory: The explicit memory of humans can be subdivided into episodic memory, which involve particular experiences, and semantic memory, which involve general truths \cite{kandel2021principles}.
This classification is analogous to our definition of specific and abstract knowledges.
Our current implementation of explicit memory is closer to the episodic memory of humans, as each memory directly corresponds to a reference text.
To improve its reasoning ability, one can try to equip \name{} with semantic memories, e.g. obtained from induction on the episodic memories.
\end{enumerate}

Besides these broad topics, there are also plenty of engineering works that can be done.
For instance, an internalized retrieval process that matches sparse attention queries with memory keys (Remark \ref{remark:self-retrieval}), sparser memory heads with routing (Remark \ref{remark:adaptive head}), memory extraction that fully preserves contexts (Remark \ref{remark:context}), compilation of the knowledge base based on machine preference (Remark \ref{remark:ref select}), reduction of the time consumption of explicit memory to be proportional to its compute overhead (Remark \ref{remark:memory time cost}), and so on.

\section*{Acknowledgement}

This work is supported by the NSFC Major Research Plan - Interpretable and General Purpose Next-generation Artificial Intelligence of China (No. 92270001).
We thank Prof. Zhiqin Xu, Prof. Zhouhan Lin, Fangrui Liu, Liangkai Hang, Ziyang Tao, Xiaoxing Wang, Mingze Wang, Yongqi Jin, Haotian He, Guanhua Huang, Yirong Hu for helpful discussions.

\bibliographystyle{plain}
\bibliography{General/References.bib}


\appendix 

\section{Cost Estimation}
\label{appendix:cost}

This section provides the calculations for Figure \ref{fig:total_cost_usage_2B}, and we equate cost with the amount of compute measured in Tflops.

Our 2.4B \name{} model is adopted as the backbone.
Recall from Section \ref{sec:model shape} that this model has shape
\begin{itemize}
\item Transformer blocks $L=44$
\item Query heads $H=40$ and key-value heads $H_{kv}=8$
\item Head dimension $d_h=80$ and hidden dimension $d=Hd_h=3200$
\item MLP width $W=d$
\item Vocabulary size as well as LM head size $n_{\text{vocab}}=60416$
\item memory layers $L_{\text{mem}}=22$, which is also the depth of the deepest memory layer.
\end{itemize}
Fix a separable knowledge $\mathcal{K}$, and represent it by one of its realizations $\mathbf{t}$ (Definition \ref{def:realization}), and assume that $\mathbf{t}$ has length $l_{\text{ref}}=128$ tokens, following the setup of our reference dataset (Section \ref{sec:ref}).
Recall from Section \ref{sec:sparse memory} that each memory has $l_{\text{mem}}=8$ tokens per memory head, and it is read by a chunk of length $l_{\text{chunk}}=64$.

Since we want to show that explicit memory is cheaper than implicit memory and RAG, it suffices to use coarse lower bounds on their costs.

\subsection{Implicit Memory}

The write cost of implicit memory or model parameters is the training compute with $\mathbf{t}$ as input.
Usually the training data of Transformer LLMs have length $2048\sim 8192$, so we assume that $\mathbf{t}$ is a subsequence of a train sample $\mathbf{t}_{\text{train}}$ with length $l_{\text{train}}=2048$.
By \cite{narayanan2021efficient}, the training compute of one step with one sample is approximately
\begin{equation*}
3 \cdot 2 \cdot \big[ L \big( l_{\text{train}} (2d^2 + 2dd_h H_{kv} + 3dW) + 2 \frac{l_{\text{train}}^2}{2} d \big) + l_{\text{train}} n_{\text{vocab}} d \big]
\end{equation*}
where $3$ means that the backward step costs twice as the forward step (and thus 3 times in total), the first $2$ means that the compute of matrix multiplication involves same amount of additions and multiplications.
The five terms in the bracket come from QO embedding, KV embedding, MLP, attention, and LM head, respectively.
The lower order terms, such as layer normalizations, are omitted.
The fraction of the compute attributable to $\mathbf{t}$ is given by
\begin{equation*}
3 \cdot 2 \cdot \big[ L \big( l_{\text{ref}} (2d^2 + 2dd_h H_{kv} + 3dW) + 2 l_{\text{ref}} \frac{l_{\text{train}}}{2} d \big) + l_{\text{ref}} n_{\text{vocab}} d \big]
\end{equation*}
Assume that one training step is sufficient for storing knowledge $\mathcal{K}$ into model parameters.
Then, the write cost is equal to the above term, and we obtain
\begin{equation*}
\text{cost}_{\text{write}} \approx 2.24 ~\text{TFlops}
\end{equation*}

Meanwhile, we lower bound the read cost by zero.
\begin{equation*}
\text{cost}_{\text{read}} \geq 0~\text{TFlops}
\end{equation*}
This lower bound is obviously correct and suits our comparison, since it makes implicit memory appear more competitive.
The difficulty in estimating the cost is that the correspondence between knowledges and parameters is not fully understood.
Nevertheless, we describe a possible way to obtain a reasonable bound.
Recall from Section \ref{sec:intro} that the model parameters suffer from the issue of knowledge traversal such that each parameter (and thus each implicit memory) is invoked during each call of the LLM.
So the read cost of each implicit memory does not depend on its usage count $n_k$, but instead on the total amount of model calls during the lifespan of this LLM.
Dividing the total amount of inference compute used by this LLM by the amount of knowledges it possesses gives an estimation of the average read cost of a knowledge.
The amount of knowledges in the LLM can be upper bounded based on the knowledge capacities measured by \cite{allenzhu2024physics}.

\subsection{Explicit Memory}

The write cost of an each explicit memory mainly comes from $L_{\text{mem}}$ self-attention layers, $L_{\text{mem}}-1$ MLP layers, and $L_{\text{mem}}$ token sparsification operations (computing the full attention matrix):
\begin{align*}
\text{cost}_{\text{write}} &= 2 \cdot \big[ L_{\text{mem}} \big( l_{\text{ref}} (2d^2 + 2dd_h H_{kv}) + 2 \frac{l_{\text{ref}}^2}{2} d \big) + (L_{\text{mem}}-1) (l_{\text{ref}} \cdot 3dW) + L_{\text{mem}} (l_{\text{ref}}^2 d) \big]\\
&\approx 0.308 ~\text{TFlops}
\end{align*}

The read cost consists of the attention to the sparse tokens of an explicit memory from the chunk that retrieves this memory:
\begin{equation*}
\text{cost}_{\text{read}} = 2 L_{\text{mem}} \cdot 2 l_{\text{chunk}} l_{\text{mem}} d \approx 1.44 \times 10^{-4} ~\text{TFlops}
\end{equation*}

\subsection{External Information}

The write cost of text retrieval-augmented generation (RAG) is set to be zero, since the reference is stored as plain text.
\begin{equation*}
\text{cost}_{\text{write}} = 0~\text{TFlops}
\end{equation*}

The read cost is the additional compute brought by the retrieved references that are inserted in the prompt.
To make RAG appear more competitive, we assume that only a chunk of the prompt or decoded text with length $l_{\text{chunk}}$ can attend to the references, and each reference can only attend to itself, which in general is not true.
Then,
\begin{align*}
\text{cost}_{\text{write}} &\geq 2 \cdot \big[ L \big( l_{\text{ref}} (2d^2 + 2dd_h H_{kv}) + 2 l_{\text{ref}} \big( \frac{l_{\text{ref}}}{2} + l_{\text{chunk}} \big) d \big) + (L-1) (l_{\text{ref}} \cdot 3dW) \big]\\
&\approx 0.624 ~\text{TFlops}
\end{align*}

\vspace{0.5em}
In summary, the total cost (TFlops) of writing and reading each separable knowledge in terms of its expected usage count $n$ is given by
\begin{equation*}
\begin{cases}
c_{\text{implicit}}(n) \geq 2.24 \\
c_{\text{explicit}}(n) = 0.308 + 0.000144 n\\
c_{\text{external}}(n) \geq 0.624 n
\end{cases}
\end{equation*}
These curves are plotted in Figure \ref{fig:total_cost_usage_2B}.
Hence, if $n\in (0.494, 13400)$, then it is optimal to store the knowledge as an explicit memory.

\begin{remark}[Knowledge retention]
One aspect not covered by Problem (\ref{eq:cost}) is the retention of knowledges in the model if its parameters are updated, e.g. due to finetuning.
Both implicit memory and explicit memory are vulnerable to parameter change.
Usually, model finetuning would include some amount of pretrain data to prevent catastrophic forgetting \cite{ouyang2022instruct}.
Similarly, if some explicit memories have already been produced, then they need to be rebuilt in order to remain readable by the updated model.
It is an interesting research direction to design a more efficient architecture such that the implicit and explicit memories are robust with respect to model updates.
\end{remark}

\section{Vector Compression}
\label{appendix:quantizer}

Regarding the vector quantizer discussed in Sections \ref{sec:sparse memory} and \ref{sec:eval-general}, we use the composite index of FAISS with index type OPQ20x80-Residual2x14-PQ8x10.
It can encode a 80-dimensional bfloat16 vector into a 14-dimensional uint8 vector, and thus its compression rate is $\frac{80\times 2}{14 \times 1} \approx 11.4$.

To train this quantizer, we sample references from our knowledge base, encode them into explicit memories by our \name{}-2B-SFT model, and feed these key-value vectors to the quantizer.
The references are sampled uniformly and independently, so the training is not biased towards the references that are retrieved by any specific evaluation task.

\section{Supplementary Evaluation Results}
\label{appendix:evaluation}

First, Table \ref{tab:eval stages} records the growth of the test scores (Table \ref{tab:basic-results}) over the three training stages: warmup, continual train, and SFT.
We believe that for future versions of \name{}, fixing the loss divergence during the warmup stage can allow the continual train stage to proceed much further (cf. Section \ref{sec:continual train}), and thus increase the performance boost of this stage.

\begin{table}[h]
\centering
\footnotesize
\rowcolors{2}{Tue-red!10}{white}
\begin{tabular}{ccccc>{\centering\arraybackslash}p{0.075\linewidth}ccc}
\toprule
                &         & \multicolumn{5}{c}{English}& \multicolumn{2}{c}{Chinese}\\ 
                                      \cmidrule(lr){3-7}                                                  \cmidrule(lr){8-9}
LLM& Avg.    & ARC-C& HellaSwag& MMLU& Winogrande& GSM8k& CEVAL &CMMLU\\
\midrule
Warmup & 42.13&	40.27&	64.57&	41.62&	61.96&	5.23&	40.12&	41.17\\
Continual train & 45.12&	42.66&	79.21&	41.81&	59.43&	6.29&	42.20&	44.21\\
- without memory & 42.89&	42.15&	66.98&	39.79&	61.80&	6.44&	39.97&	43.13\\
SFT & 63.31& 58.11& 80.51& 59.68& 74.51& 52.84& 59.29& 58.24\\
- without memory & 60.80& 57.42& 73.14& 57.29& 74.35& 51.33& 56.32& 55.72\\
\bottomrule
\end{tabular}
\caption{Performance of \name{}-2B at different stages of training.
The setup of the evaluation tasks is the same as in Table \ref{tab:basic-results}.}
\label{tab:eval stages}
\end{table}

Next, recall that for the evaluations in Section \ref{sec:eval-general}, a filter is included in the retrieval process to prevent copying, which removes references that overlap too much with the evaluation question.
The filtering threshold should lie between $100\%$ and the usual level of overlap between two related but distinct texts, and we set it to $2/3$ in Table \ref{tab:basic-results}.
Table \ref{tab:filter threshold} records the impact of the filtering threshold on the test scores.
The scores are stable for most tasks, indicating that their questions do not appear in our knowledge basis.

\begin{table}[h]
\centering
\footnotesize
\rowcolors{2}{Tue-red!10}{white}
\begin{tabular}{ccccc>{\centering\arraybackslash}p{0.075\linewidth}ccc}
\toprule
Threshold& Avg.    & ARC-C& HellaSwag& MMLU& Winogrande& GSM8k& CEVAL &CMMLU\\
\midrule
no filter & 63.71&	58.11&	83.37&	59.65&	74.51&	52.84&	59.29&	58.22\\
80\% & 63.62&	58.11&	82.69&	59.65&	74.51&	52.84&	59.29&	58.24\\
2/3 & 63.31& 58.11& 80.51& 59.68& 74.51& 52.84& 59.29& 58.24\\
without memory & 60.80& 57.42& 73.14& 57.29& 74.35& 51.33& 56.32& 55.72\\
\bottomrule
\end{tabular}
\caption{Influence of the filtering threshold on the test scores in Table \ref{tab:basic-results}.}
\label{tab:filter threshold}
\end{table}

Finally, Table \ref{tab:few-shot vs 0-shot} studies the influence of the few-shot prompts on the benchmark tasks.
Recall that the number of few-shot examples for each task is ARC-C (25), HellaSwag (10), MMLU (5), Winogrande (5), GSM8k (5) as in HuggingFace OpenLLM Leaderboard \cite{open-llm-leaderboard}, and we also adopt CEVAL (5), CMMLU (5).
Interestingly, the boost from explicit memory increases from $2.51\%$ to $3.70\%$ as we switch to 0-shot.

\begin{table}[h]
\centering
\footnotesize
\rowcolors{2}{Tue-red!10}{white}
\begin{tabular}{ccccc>{\centering\arraybackslash}p{0.075\linewidth}ccc}
\toprule
Mode& Avg.    & ARC-C& HellaSwag& MMLU& Winogrande& GSM8k& CEVAL &CMMLU\\
\midrule
Few-shot & 63.31& 58.11& 80.51& 59.68& 74.51& 52.84& 59.29& 58.24\\
- without memory & 60.80& 57.42& 73.14& 57.29& 74.35& 51.33& 56.32& 55.72\\
0-shot & 58.23&	58.79&	83.29&	60.53&	75.85&	13.50&	57.95&	57.74\\
- without memory & 54.54&	57.34&	73.15&	58.59&	74.98&	10.46&	54.53&	54.26\\
\bottomrule
\end{tabular}
\caption{Few-shot versus 0-shot for the benchmark tasks in Table \ref{tab:basic-results}.}
\label{tab:few-shot vs 0-shot}
\end{table}

\end{document}

%% file: General/Preamble.tex
\usepackage[utf8]{inputenc} 
\usepackage{amsmath, amsthm, amssymb, amsfonts} 
\usepackage{lipsum} 

\usepackage{graphicx} 
\usepackage{float} 
\usepackage{caption} 
\usepackage{subcaption} 

\usepackage{tabularx} 
\usepackage{tabu} 
\usepackage{booktabs} 

\usepackage[nottoc,numbib]{tocbibind} 
\usepackage[margin = 2.5cm]{geometry} 

\usepackage{microtype} 
\usepackage{titlesec} 
\usepackage{titletoc} 
\usepackage{appendix} 
\usepackage{fancyhdr} 
\usepackage[shortlabels]{enumitem} 

\usepackage{hyperref} 
\usepackage[noabbrev, capitalise]{cleveref} 
\usepackage{authblk}

\usepackage{xcolor} 

\usepackage{tikz} 
\usetikzlibrary{positioning} 
\usetikzlibrary{calc} 

\makeatletter
\newcommand{\gettikzxy}[3]{%
  \tikz@scan@one@point\pgfutil@firstofone#1\relax
  \edef#2{\the\pgf@x}%
  \edef#3{\the\pgf@y}%
}
\makeatother

\theoremstyle{definition}

\newtheorem{definition}{Definition}
\newtheorem{remark}{Remark}

\newtheorem{claim}{Claim}
\newtheorem{example}{Example}
\newtheorem{assumption}{Assumption}

\newcommand{\prob}{\mathbb{P}}

\newcommand{\K}{\mathcal{K}}

\newcommand{\x}{\mathbf{x}}

%% file: General/Settings.tex
\newcommand{\name}{Memory$^3$}

\definecolor{Tue-red}{RGB}{0, 134, 139}

\hypersetup{
    colorlinks=true,
    linkcolor=Tue-red,
    urlcolor=Tue-red,
    citecolor = Tue-red
    }
\counterwithout{equation}{section} 
\counterwithout{figure}{section} 
\counterwithout{table}{section} 
\makeatletter
\let\c@figure\c@table
\makeatother

\usepackage{makecell}

\titleformat{\section}{\sffamily\color{Tue-red}\Large\bfseries}{\thesection\enskip\color{gray}\textbar\enskip}{0cm}{} 

\titleformat{\subsection}{\sffamily\color{Tue-red}\large\bfseries}{\thesubsection\enskip\color{gray}\textbar\enskip}{0cm}{} 

\titleformat{\subsubsection}{\sffamily\color{Tue-red}\bfseries}{\thesubsubsection\enskip\color{gray}\textbar\enskip}{0cm}{} 

\DeclareCaptionFont{Tue-red}{\color{Tue-red}}
\usepackage{mlmodern}

